%% file: iclr2027_conference.tex
\documentclass{article} %
\usepackage{iclr2027_conference,times}

\input{math_commands.tex}

\usepackage{lipsum}
\usepackage{wrapfig}
\usepackage{pifont} %
\usepackage{adjustbox}
\usepackage{booktabs}
\usepackage{hyperref}
\usepackage{url}
\usepackage{xcolor}

\usepackage{amsmath,amsfonts,amssymb,dsfont}
\usepackage{amsthm}
\usepackage{aliascnt}
\usepackage{placeins}
\usepackage{float}
\usepackage[skip=4pt]{caption}
\usepackage{enumitem}
\usepackage{thmtools, thm-restate}

\usepackage[capitalize,noabbrev]{cleveref}
\crefname{theorem}{Theorem}{Theorems}
\crefname{lemma}{Lemma}{Lemmas}
\crefname{corollary}{Corollary}{Corollaries}
\crefname{proposition}{Proposition}{Propositions}

\usepackage{mathtools}
\usepackage{xparse}
\DeclarePairedDelimiterX{\fargs}[2]{(}{)}{#1\,\delimsize\|\,#2}

\newcommand{\authsep}{\kern-0.06em,\quad}

\RenewDocumentCommand{\KL}{gg}{%
  D_{\mathrm{KL}}\IfNoValueF{#1}{\fargs{#1}{#2}}%
}
\NewDocumentCommand{\df}{gg}{%
  D_{f}\IfNoValueF{#1}{\fargs{#1}{#2}}%
}
\NewDocumentCommand{\chisq}{gg}{%
  D_{\chi^2}\IfNoValueF{#1}{\fargs{#1}{#2}}%
}
\NewDocumentCommand{\js}{gg}{%
  D_{\mathrm{JS}}\IfNoValueF{#1}{\fargs{#1}{#2}}%
}
\NewDocumentCommand{\TV}{gg}{%
  D_{\mathrm{TV}}\IfNoValueF{#1}{\fargs{#1}{#2}}%
}
\NewDocumentCommand{\Hellinger}{gg}{%
  D_{\mathrm{H}}\IfNoValueF{#1}{\fargs{#1}{#2}}%
}
\newcommand{\sg}{\mathop{\mathrm{sg}}}

\newcommand*\samethanks[1][\value{footnote}]{\footnotemark[#1]}

\newcommand{\hide}[1]{\textcolor{red}{OMITTING TEXT}}

\title{Does self-distillation improve reasoning?}

\author{%
  Jason Ken Adhinarta$^{1}$\thanks{Equal contribution. } \quad
  Raina Wu$^{1}$\samethanks[1] \quad %
  Ashutosh Sharma$^{2}$ \quad 
  Avinash Balakrishnan$^{2}$\\[0.8ex]
  \bfseries Rameswar Panda$^{2}$\quad
  Polina Golland$^{1}$\quad
  Yoon Kim$^{1}$\quad
  Zhang-Wei Hong$^{2}$ \\[1.2ex]
  \normalfont
  $^{1}$Massachusetts Institute of Technology \qquad $^{2}$MIT-IBM Computing Research Lab
}

\iclrfinalcopy %

\ificlrfinal
  \newcommand{\coderepo}{\url{https://github.com/sleepypandu/does-sd-improve-reasoning}}
\else
  \newcommand{\coderepo}{\url{https://anonymous.4open.science/r/does-sd-improve-reasoning/}}
\fi
\begin{document}

\maketitle
\lhead{Preprint}   %

\begin{abstract}

Self-distillation is an emerging approach that trains a model to imitate its own predictions when conditioned on privileged information, such as correct answers or environmental feedback. Recent methods have reported better sample efficiency and generalization than reinforcement learning (RL), but it remains unclear whether these gains extend beyond in-distribution evaluations and whether they stem from privileged distillation. In this paper, we revisit two representative methods, Self-Distillation Policy Optimization (SDPO) and On-Policy Self-Distillation (OPSD), across code, science, tool use, and math reasoning tasks. We find that their reported gains do not consistently generalize to unseen problems and are largely explained by overfitting or transferring existing reasoning behavior, rather than effectively using privileged information. Moreover, self-distillation gradients show weak alignment with policy gradients that directly optimize expected reward, questioning whether they provide an effective direction for policy improvement. Our experiments across Qwen3 models (1.7B–14B) and OLMo-3-7B suggest that effectively leveraging privileged information to improve reasoning remains an open challenge.
Our code is available at \coderepo. %
\end{abstract}

\section{Introduction}

\begin{wrapfigure}[21]{r}{0.30\textwidth}
  \centering
  \vspace{-1em}
  \includegraphics[width=0.95\linewidth]{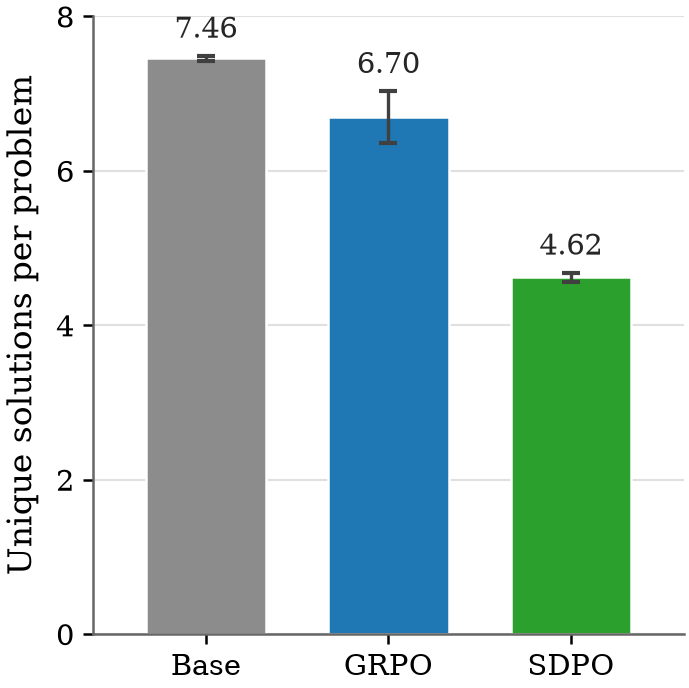}
  \caption{\textbf{SDPO reduces solution diversity on LiveCodeBench v6.} AST-unique solutions per training problem (from $8$ samples) for Qwen3-14B before (Base) and after $200$ steps of GRPO or SDPO (mean $\pm$ std, four seeds). Diversity collapse suggests memorization.}
  \label{fig:sdpolcbunique}
\end{wrapfigure}

Reinforcement learning with verifiable rewards (RLVR) has become a standard approach for finetuning large language models (LLMs) on reasoning tasks such as mathematics and coding \citep{rlvr,deepseekr1}. As responses in these domains are automatically verifiable, scaling RLVR is straightforward. However, typical RLVR pipelines use binary rewards that indicate only whether the final answer is correct or whether all unit tests pass. Such rewards provide no information on how attempted solutions should be revised. Process reward models score individual reasoning steps \citep{mathshepherd,prime}, but require training additional models.

A recently popularized alternative to RLVR is self-distillation, which requires neither an external stronger model nor an explicit reward function. It exploits the fact that reasoning tasks often come with feedback far richer than the binary outcome that RLVR consumes, such as reference solutions, critiques, error messages, or failing unit tests. In self-distillation, the model samples a trace and is trained to match, at each token along that trace, the distribution of a (teacher) copy of itself that additionally conditions on this privileged information \citep{sdpo,sdft,opsd}. The self-conditioned teacher converts rich but unstructured feedback into a dense, per-token training signal.

However, prior work questions the validity of self-distillation objectives as a mechanism to improve the student policy. The self-conditioned teacher's outputs do not express uncertainty through tokens such as ``wait'', ``hmmm'', and ``perhaps,'' so training against it suppresses the model's ability to correct past mistakes as it reasons and shortens the student's traces in ways that hurt performance \citep{kim2026degrade,rlcsd,sdthinkbad}. Moreover, the self-conditioned teacher is unable to consistently distinguish between correct and incorrect student generations \citep{microsoft,pinotvalue}. Together, these results raise a more fundamental question: \textit{does self-distillation extract useful guidance from privileged information, or are its reported gains driven by simpler mechanisms?}

\begin{figure}[t]
\begin{center}
\includegraphics[width=1.0\textwidth]{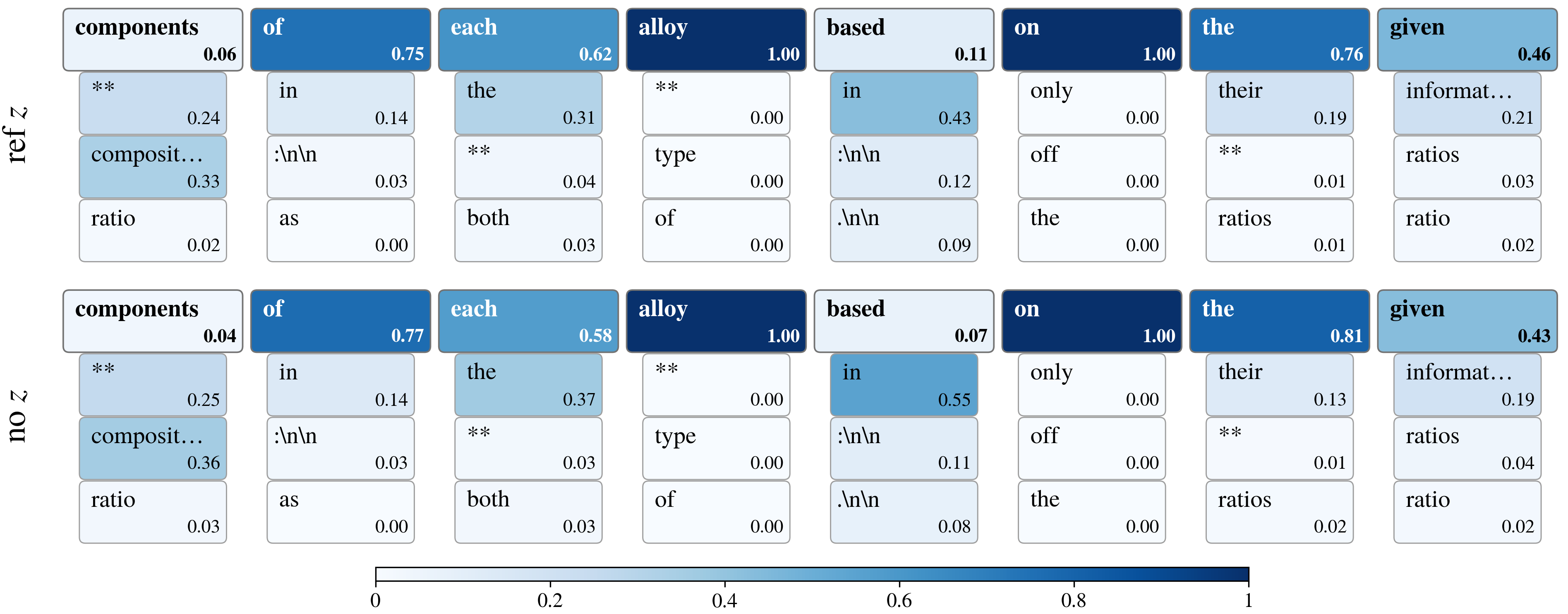}
\end{center}
\caption{\textbf{The reference solution barely changes the self-distillation teacher.} Teacher log-probabilities (thinking on) along a Qwen3-8B thinking-off student rollout on an OpenThoughts problem, with the problem's reference solution in context (ref $z$, top) and with no solution (no $z$, bottom). In each panel, the first row is the sampled token and the three rows below are the student's next three most likely tokens. The two teachers assign very similar probabilities, suggesting that OPSD distillation effects do not significantly depend on $z$ (\cref{subsec:opsd}).}
\label{fig:teaser}
\end{figure}

We revisit this question through careful reproductions of self-distillation policy optimization (SDPO) \citep{sdpo} and on-policy self-distillation (OPSD) \citep{opsd}. We leverage Qwen3 \citep{qwen3technicalreport} and Olmo-3 \citep{olmo2025olmo} models to ensure our observations hold over a range of model sizes and families. We find that the reported gains from privileged-information self-distillation in reasoning domains are weaker than they appear:
\begin{itemize}[leftmargin=*]
\item \textbf{SDPO's coding gains do not transfer to unseen problems} (\cref{subsec:sdpo}).
SDPO's LiveCodeBench v6 \citep{lcb} results train and evaluate on the same set of coding problems. Correcting for this reveals that SDPO overfits to its training problems, reducing the diversity of its solutions (\cref{fig:sdpolcbunique}), and degrades on held-out problems while policy-gradient training on the same verifiable reward generalizes.

\item \textbf{SDPO's gains in science reasoning and tool use do not generalize out of domain} (\cref{subsec:sft}).
Across five science-reasoning \citep{sciknoweval} and tool-use \citep{tooluse} domains, SDPO and GRPO in-domain performances are similar, but in four of five it lowers accuracy on the held-out domains, while GRPO preserves it.

\item \textbf{OPSD transfers thinking-mode behavior, not privileged information} (\cref{subsec:opsd}).
In math reasoning, privileged information barely changes the self-conditioned teacher's token distribution (\cref{fig:teaser}), and distilling thinking-mode behavior has a substantially larger effect than privileged information within OPSD's reported results.

\item \textbf{Self-distillation gradients do not align with policy gradients} (\cref{subsec:orthogonality}). Prior work motivates self-distillation as approximate policy improvement by arguing that its gradient aligns with the policy gradient \citep{sdft}. This alignment matters because a small update improves expected reward only to the extent that it follows the policy gradient (\cref{subsec:equivalence}). We therefore measure how well each method's training gradient aligns with the policy gradient on held-out problems. On Qwen3 models, OPSD and SDPO gradients are consistently less aligned than GRPO gradients computed on the same training problems and often point away from the held-out policy gradient altogether.
\end{itemize}

\section{Related works}

\textbf{Reinforcement learning from verifiable rewards (RLVR).} The goal of reinforcement learning is to maximize the expectation of reward under a policy's responses. For verifiable tasks in math reasoning and coding, such rewards are assigned by verifying the equivalence between expressions or checking whether a proposed program passes all test-cases. RLVR for LLMs is typically done using policy gradient algorithms \citep{reinforce,policygradient}. Although PPO \citep{ppo} was used in foundational work \citep{instructgpt,helpfulharmless}, critic-free methods such as GRPO \citep{grpo} and RLOO \citep{Kool2019Buy4R,ahmadian2024basicsrevisitingreinforcestyle} have become popular due to their simplicity, despite feedback only being given at the trajectory level.

\textbf{On-policy distillation (OPD).} Unlike classic distillation, which tries to match teacher-generated outputs \citep{hinton2015distillingknowledgeneuralnetwork}, on-policy distillation seeks to boost performance by minimizing per-token divergences between student logits $\pi_\theta(\cdot|x,y_{<t})$ and teacher logits $\pi_T(\cdot | x,y_{<t})$ along prefixes sampled from the \textit{student policy} \citep{lu2025onpolicydistillation,deepseekv4,gu2023minillm}. OPD has been shown to be effective in post-training small-models \citep{qwen3technicalreport} and merging the capabilities of multiple domain-specialist expert models \citep{deepseekv4,kimik3,mopd}.

\textbf{Self-distillation and privileged information.} Instead of relying on an external stronger model, self-distillation \citep{sdft,opsd,sdpo} parameterizes the teacher distribution as $\pi_T(\cdot | x,y_{<t})=\pi_{\theta}(\cdot | x, z, y_{<t})$, with privileged information $z$ which is derived from a reference solution or environmental feedback. This purportedly moves the student towards a higher-reward policy, transferring the reference solution's reasoning content.

\textbf{Limitations of self-distillation.} Prior work documents several failure modes of self-distillation: SDPO generalizes poorly out of distribution and in continual learning settings \citep{sdpobadoodifbadteachersignal}, self-distillation degrades thinking models by suppressing backtracking behavior and collapsing response lengths \citep{sdthinkbad,microsoft}, and the self-conditioned teacher fails to distinguish correct from incorrect student solutions \citep{pinotvalue}. Concurrently, \citet{opsdscoop} show that OPSD's gains stem from thinking-mode transfer rather than privileged information. Methods that retain GRPO and use the teacher only to reweight advantages such as \citet{rlcsd} fall outside our scope, as we study the self-distillation objective in isolation. We provide an extended discussion of existing work on the limitations of self-distillation, and of how our paper differs from each, in \cref{app:extended_related_works}.

\section{Preliminaries}
We introduce the baseline training objectives, then show in \cref{subsec:equivalence} that under an optimal teacher all three share the same update direction, which motivates the gradient analysis of \cref{subsec:orthogonality}.

\textbf{Notation.}
Let $\mathcal{X}$, $\mathcal{Y}$, and $\mathcal{Z}$ denote the spaces of problems, solutions, and privileged information respectively. We assume all three are finite, as is the case for language models with a finite vocabulary and a maximum sequence length. For $y \in \mathcal{Y}$, we write $y_t$ for its $t$-th token, $y_{<t}$ for its first $t-1$ tokens, and $|y|$ for its length. For distributions $p, q$ over $\mathcal{Y}$, $\KL{p}{q} = \sum_{y \in \mathcal{Y}} p(y) \log \frac{p(y)}{q(y)}$.

\textbf{Supervised fine-tuning (SFT).}
Given a distribution over problems $x\sim \rho(\cdot)$ and a distribution over solutions $y\sim \pi_T(\cdot | x)$, supervised finetuning optimizes $\pi_{\theta}(\cdot | x)$ to maximize the likelihood of $y$:
\begin{equation}
    \mathcal{L}_{\operatorname{SFT}}(\theta) = -\mathbb{E}_{x\sim \rho(\cdot),y\sim\pi_T(y| x)} \left[\sum_{t=1}^{|y|} \log \pi_{\theta}\!\left(y_t \mid x, y_{<t}\right)\right].
\end{equation}
Note that optimizing the objective above is equivalent to minimizing $\KL$ from the data distribution to the policy, namely, $\mathbb{E}_{x \sim \rho(\cdot)}[\KL{\pi_T(\cdot | x)}{\pi_\theta(\cdot | x)}]$.

\textbf{Group Relative Policy Optimization (GRPO).} Given a distribution over problems $x \sim \rho(\cdot)$ and a reward function $R:\mathcal{X}\times \mathcal{Y}\rightarrow \mathbb{R}$, for a given $x$, GRPO samples a group of $K$ rollouts $y_1,\dots, y_K \sim \pi_{\theta}(\cdot | x)$ from the student and computes their rewards $r_i = R(x,y_i)$. The loss is then computed as
\begin{equation}
    \mathcal{L}_{\operatorname{GRPO}}(\theta) = -\frac{1}{\sum_{i=1}^{K} |y_i|} \sum_{i=1}^{K} \sum_{t=1}^{|y_i|} \min\left( w_{i,t}(\theta)\,\hat{A}_i,\; \operatorname{clip}\left(w_{i,t}(\theta), 1-\varepsilon, 1+\varepsilon\right)\hat{A}_i \right), 
\end{equation} %
where $w_{i,t}(\theta) = \frac{\pi_\theta(y_{i,t} \mid x, y_{i,<t})}{\pi_{\theta_{\mathrm{old}}}(y_{i,t} \mid x, y_{i,<t})}$ is the per-token importance ratio and $\hat{A}_i = \frac{r_i - \bar{r}}{\operatorname{std}(r)}$ is a sequence-level $z$-score normalized advantage shared by all tokens in rollout $i$. GRPO is an instantiation of policy-gradients which aims to maximize $\mathbb{E}_{x\sim \rho(\cdot),y\sim \pi_\theta(\cdot | x)}[R(x,y)]$ by using the following gradient update $\mathbb{E}_{x\sim \rho(\cdot),y\sim \pi_\theta(\cdot | x)}[R(x,y)\nabla_{\theta} \log \pi_\theta(y | x)]$.

\textbf{Privileged information self-distillation.}
Given a distribution over problems $x \sim \rho(\cdot)$ and a distribution over privileged information $z \sim \gamma(\cdot | x)$, we construct a privileged teacher $\pi_T(\cdot | x)=\pi_{\theta}(\cdot |x, z)$ for problems $x$ for some sampled $z$. \mbox{\citet{opsd}} use a fixed reference solution in the dataset as $z$, while \citet{sdpo} use a successful sibling rollout and environmental feedback if available. The loss is then computed as \begin{equation}
    \mathcal{L}_{\operatorname{SD}} = \mathbb{E}_{x\sim \rho(\cdot),y \sim \operatorname{sg}(\pi_{\theta}(\cdot \mid x)),z\sim \gamma(\cdot| x)} \left[\sum_{t=1}^{|y|} D\left( \pi_{\theta}(\cdot \mid x, y_{<t}), \operatorname{sg}(\pi_\theta(\cdot \mid x, y_{<t}, z)) \right)\right],
    \label{eq:sd_estimator}
\end{equation}
where $\operatorname{sg}$ is the stop-grad operator and $D$ represents an information divergence such as forward KL, reverse KL, or Jensen-Shannon divergence. The sampler $y\sim \pi_\theta(\cdot | x)$ is typically not differentiated through due to stability reasons \citep{gkd,sdpo}. Common tweaks include an EMA-teacher \citep{sdpo,sdft}, top-$k$ logit approximations \citep{sdpo}, and per-token clipping \citep{opsd}. Notably, by choosing $D(p, q) = \KL{p}{q}$, we have that $ \mathcal{L}_{\operatorname{SD}}$ is \textit{equal in value} to $ \mathbb{E}_{x \sim \rho(\cdot),\, z \sim \gamma(\cdot \mid x)} \left[\KL{\pi_\theta(\cdot \mid x)}{\pi_T(\cdot \mid x, z)}\right]$. 

The two methods we consider, SDPO \cite{sdpo} and OPSD \cite{opsd}, differ primarily in their instantiation of \cref{eq:sd_estimator} and their use of privileged information: the former uses reverse KL restricted to the student's top $k$ tokens and conditions the teacher on a successful sibling rollout and environment feedback and the latter implements forward KL with per-position full-vocabulary divergence clipping and takes $z$ as a reference solution.

\subsection{Equivalence between Policy Gradient, SFT, and Distillation}\label{subsec:equivalence}

Policy gradient improves the policy by ascending expected reward, whereas self-distillation moves the student towards a teacher, so whether it improves depends on the teacher. Self-distillation implicitly assumes that the self-conditioned teacher approximates the KL-regularized optimal policy relative to the current policy, $\pi^*(y|x) \propto \pi_\theta(y|x)\exp(R(x,y)/\beta)$, which we call the optimal teacher \citep{sdft}. We show that for this teacher, the policy gradient, SFT on samples from $\pi^*$, and on-policy distillation towards $\pi^*$ share the same update direction. Hence, if the assumption holds, self-distillation gradients should align with the policy gradient, which we test in \cref{subsec:orthogonality}.

\begin{restatable}{proposition}{equivprop}\label{prop:equivalence}
Consider the single-prompt setting, with a fixed $x \sim \rho(\cdot)$. Fix any policy $\pi_\theta(\cdot|x)$ over $\mathcal{Y}$, a binary reward function $R:\mathcal{X}\times\mathcal{Y} \rightarrow \{0,1\}$, and $\beta > 0$, and define the optimal teacher $\pi^*(y|x)=\frac{1}{Z(x)} \pi_\theta(y|x) \exp{(R(x,y)/\beta)}$, which maximizes $\mathbb{E}_{y\sim \pi(\cdot|x)}[R(x,y)]-\beta \KL{\pi(\cdot|x)}{\pi_\theta(\cdot|x)}$ over policies $\pi$. Since $x$ is fixed, we suppress it from the notation, writing $\pi_\theta(y)$, $R(y)$, $Z$, and $\pi^*(y)=\frac{1}{Z}\pi_\theta(y)\exp{(R(y)/\beta)}$. Then, we have that:
\begin{align} \label{eq:prop3.1}
\frac{\nabla_\theta \mathbb{E}_{y\sim \pi_\theta}[R(y)]}{\|\nabla_\theta \mathbb{E}_{y\sim \pi_\theta}[R(y)]\|_2}
= -\frac{\nabla_\theta \KL{\sg{(\pi^*)}}{\pi_\theta}}{\|\nabla_\theta \KL{\sg{(\pi^*)}}{\pi_\theta}\|_2}  = -\frac{\nabla_\theta \KL{\pi_\theta}{\sg{(\pi^*)}}}{\|\nabla_\theta \KL{\pi_\theta}{\sg{(\pi^*)}}\|_2},
\end{align}
whenever $\nabla_\theta \mathbb{E}_{y\sim \pi_\theta}[R(y)]$ is non-zero.
\end{restatable}

We defer the proof to \cref{app:propproof}.

Crucially, the self-conditioned teacher must stay close to the student. RL generalizes better than SFT on reasoning tasks \citep{sftmemorizesrlgeneralizes}, plausibly since on-policy RL induces a smaller KL shift from the base model \citep{razor}; a teacher far from $\pi_\theta$ forfeits this small-KL-shift property.

\section{Experimental results} \label{sec:exp}
We now examine whether the reported gains of prior self-distillation works reflect the generalization performance of policy gradients (\cref{subsec:sdpo}), require reasoning behavior (\cref{subsec:sft}), transfer via privileged information (\cref{subsec:opsd}), and have gradients that align with the policy gradient as \cref{prop:equivalence} would require of an optimal teacher (\cref{subsec:orthogonality}).

\subsection{SDPO does not generalize to new coding problems} \label{subsec:sdpo}
\citet{microsoft} suggest that the self-conditioned teacher is heavily biased towards the reference solution given in context in a way that does not generalize, which makes the coding gains reported by SDPO \citep{sdpo} surprising. We observe that SDPO's coding evaluations reuse its training problems, thus not cleanly measuring generalization. We hypothesize that SDPO's coding gains stem from this discrepancy.

SDPO builds its teacher from privileged information available only at training time. On coding tasks, the teacher sees a successful student rollout from the same batch along with environment feedback on execution errors on test cases. They claim that SDPO improves $7.6\%$ over GRPO baselines in accuracy on LiveCodeBench v6 \citep{lcb} evaluations. We construct a clean split of the LCB data, reproduce SDPO overall gains on their data split, and observe that SDPO does not lead to performance gains when evaluated on an unseen test set. %

\begin{figure}[h]
\begin{center}
\includegraphics[width=0.9\textwidth]{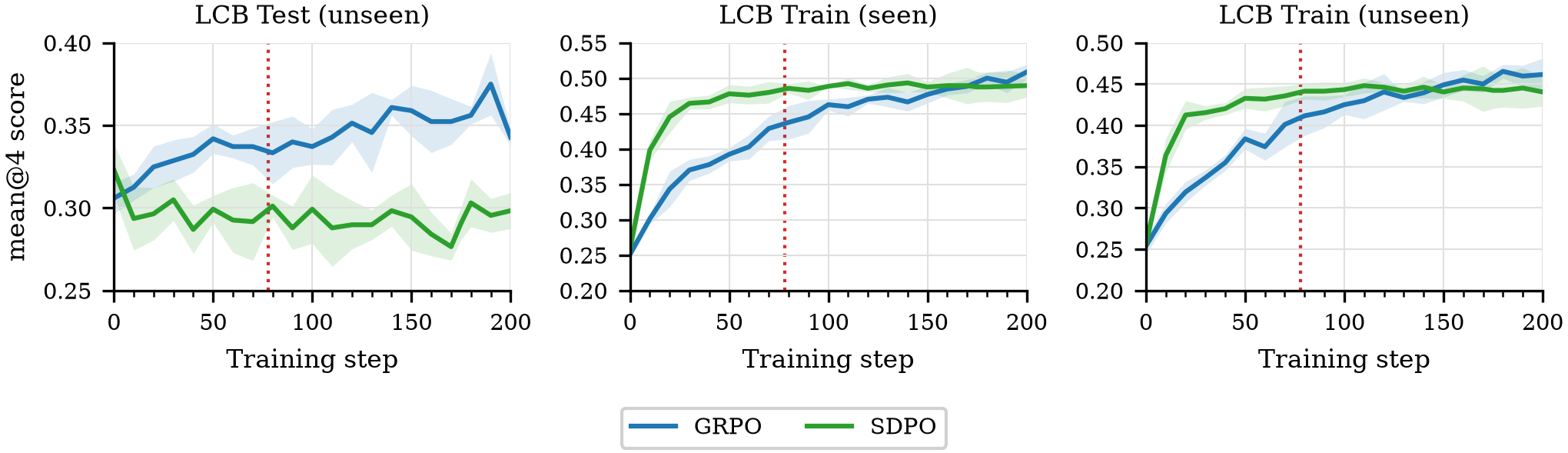}
\end{center}
\caption{Qwen3-8B SDPO vs.\ GRPO on LiveCodeBench v6 across four seeds, scored as the fraction of problems with no failing test cases (mean $\pm$ std). \textsc{train\_seen} and \textsc{train\_unseen} evaluate training problems on seen and on held-out test cases respectively; \textsc{test\_unseen} evaluates held-out problems. The $20000$ training generations reported by \citet{sdpo} is marked in red.}
\label{fig:sdpolcb}
\end{figure}

\textbf{Setup.}
LiveCodeBench v6 pairs each competitive programming problem with a set of test cases. Concretely, the original work trains on \textit{all} problems with half of each problem's test cases removed, and evaluates on the unmodified dataset. To separate memorization from generalization, we instead split the problems into disjoint, equal-size sets $\mathcal{P}_{\mathrm{train}}$ and $\mathcal{P}_{\mathrm{test}}$, giving three evaluation groups: \textsc{train\_seen} (test cases used during training), \textsc{train\_unseen} (held-out test cases of problems in $\mathcal{P}_{\mathrm{train}}$), and \textsc{test\_unseen} (all test cases of problems in $\mathcal{P}_{\mathrm{test}}$).

\textbf{Results.} We train Qwen3-8B, Qwen3-14B, and Olmo-3-7B-Instruct on this split with SDPO and GRPO for 200 steps, four seeds each, otherwise following the setup of \citet{sdpo}; we provide our hyperparameters in \cref{subsec:sdpo_implementation}. \Cref{fig:sdpolcb} shows Qwen3-8B; results for the other two models and tables of accuracies are in \cref{app:sdpolcbcont}.

\textbf{SDPO does not improve on new problems.} Over 200 steps, GRPO improves \textsc{test\_unseen} by $3.7$, $4.9$, and $2.6$ points on Qwen3-8B, Qwen3-14B, and Olmo-3-7B-Instruct, whereas SDPO changes it by $-2.5$, $-0.5$, and $-9.9$ points. The disparity is largest on Olmo-3-7B-Instruct, where SDPO ends $13.1$ points below GRPO ($21.0$ vs.\ $34.1$). In summary, we find that mimicking a privileged teacher does not generalize to new problems.

\textbf{SDPO does generalize to held-out test cases.} SDPO's gains on \textsc{train\_seen} largely carry over to \textsc{train\_unseen} ($+23.1$ vs.\ $+18.4$ on Qwen3-8B, $+22.5$ vs.\ $+17.2$ on Qwen3-14B, and $+18.3$ vs.\ $+16.7$ on Olmo-3-7B-Instruct), retaining 76--91\% of the gain compared with 81--103\% for GRPO.  SDPO therefore overfits at the level of training problems instead of at the level of unit tests, and produces fewer distinct solutions per training problem than GRPO (\cref{fig:sdpolcbunique}). We hypothesize this is because a problem's test cases check the same specification, so a program that passes the seen test cases is usually correct and passes the held-out ones too. This reconciles our results with those of \citet{sdpo}; indeed, evaluating on \textsc{train\_seen} and \textsc{train\_unseen} at 80 steps across all three models, we find that SDPO beats GRPO.

\textbf{On Qwen3 models, GRPO overtakes SDPO on training problems with longer training.} SDPO improves faster early on but then slows, while GRPO keeps improving and overtakes it after $\sim$160 steps on both \textsc{train\_seen} and \textsc{train\_unseen}. At 200 steps on Qwen3-8B, GRPO scores $50.9$ on \textsc{train\_seen} and $46.2$ on \textsc{train\_unseen}, versus $48.9$ and $44.0$ for SDPO; on Qwen3-14B, GRPO scores $54.1$ and $49.2$, versus $51.9$ and $45.4$. The $\sim$80-step comparison of \citet{sdpo} therefore does not reflect performance at convergence---at 80 steps, GRPO is still improving and SDPO's advantage disappears with further training. In Olmo-3-7B-Instruct, SDPO continues to outperform GRPO on training problems.

\subsection{SDPO does not generalize OOD in science reasoning and tool use} \label{subsec:sft}
Given that SDPO overfits on coding problems, we also examine its reported gains in science reasoning and tool use, where \citet{sdpo} find that SDPO outperforms GRPO. We ask how broadly this advantage holds and whether it remains in out-of-domain settings.

\begin{table}[h]
  \centering
  \footnotesize
  \setlength{\tabcolsep}{5pt}
  \caption{Cross-domain validation accuracy (\texttt{avg@16}) of Qwen3-8B after $200$ training steps (mean$\pm$std over four seeds). Rows vary across training domains while columns vary across evaluation domains. OOD change is defined as the change in mean accuracy over all non-training domains.}
  \label{tab:cross_domain}
  \begin{adjustbox}{max width=\textwidth}
  \begin{tabular}{lcccccc}
    \toprule
     & \multicolumn{5}{c}{\textbf{Evaluated on}} & \\
    \cmidrule(lr){2-6}
    \textbf{Trained on} & \textbf{Biology} & \textbf{Chemistry} & \textbf{Material} & \textbf{Physics} & \textbf{Tool use} & \textbf{OOD change} \\
    \midrule
    Base & 30.5$\pm$1.6 & 41.3$\pm$0.6 & 59.0$\pm$0.5 & 59.2$\pm$0.7 & 57.9$\pm$0.4 & \\
    \midrule
    \multicolumn{7}{l}{\textit{SDPO}} \\
    \quad Biology   & \textbf{52.5$\pm$4.9} & 36.7$\pm$1.5 & 32.6$\pm$1.0 & 37.8$\pm$2.9 & 58.9$\pm$0.9 & $-12.8\pm0.8$ \\
    \quad Chemistry & 27.1$\pm$2.1 & \textbf{77.7$\pm$0.6} & 36.0$\pm$1.9 & 46.6$\pm$1.7 & 30.6$\pm$19.6 & $-16.6\pm4.4$ \\
    \quad Material  & 33.0$\pm$2.7 & 47.0$\pm$1.5 & 73.5$\pm$3.3 & 45.3$\pm$8.8 & 35.1$\pm$7.0 & $-7.1\pm4.0$ \\
    \quad Physics   & 33.7$\pm$1.8 & 45.2$\pm$2.0 & 47.2$\pm$4.9 & \textbf{73.1$\pm$4.3} & 12.5$\pm$16.6 & $-12.5\pm4.6$ \\
    \quad Tool use  & 38.1$\pm$2.8 & 46.9$\pm$1.5 & 56.6$\pm$5.3 & 51.7$\pm$2.5 & 53.5$\pm$4.5 & $+0.9\pm1.8$ \\
    \midrule
    \multicolumn{7}{l}{\textit{GRPO}} \\
    \quad Biology   & 51.0$\pm$2.2 & 42.1$\pm$0.8 & 60.5$\pm$2.4 & 59.9$\pm$2.0 & 59.6$\pm$1.1 & $+1.2\pm0.9$ \\
    \quad Chemistry & 31.1$\pm$2.2 & 68.9$\pm$1.1 & 62.3$\pm$1.5 & 60.3$\pm$0.9 & 59.1$\pm$0.8 & $\mathbf{+1.6\pm0.4}$ \\
    \quad Material  & 26.9$\pm$1.3 & 46.1$\pm$1.8 & \textbf{76.2$\pm$1.6} & 61.7$\pm$0.8 & 59.1$\pm$0.4 & $+1.2\pm0.7$ \\
    \quad Physics   & 27.0$\pm$1.7 & 43.6$\pm$1.1 & 63.2$\pm$1.5 & 72.4$\pm$5.2 & 58.4$\pm$0.8 & $+0.9\pm0.1$ \\
    \quad Tool use  & 30.4$\pm$2.2 & 41.9$\pm$0.4 & 54.3$\pm$1.6 & 54.8$\pm$1.8 & \textbf{66.3$\pm$0.7} & $-2.1\pm1.2$ \\
    \bottomrule
  \end{tabular}
  \end{adjustbox}
\end{table}

\textbf{Setup.}
\citet{sdpo} evaluate SDPO on SciKnowEval multiple-choice science questions \citep{sciknoweval} and ToolAlpaca \citep{tooluse}, with teachers conditioned on correct neighboring solutions. Unlike more difficult math and coding tasks, where reaching the answer requires multi-step derivation, answers in these domains can be directly produced by choosing among four options or emitting an API call whose structure is largely determined by the prompt. We train Qwen3-8B with SDPO and GRPO on each domain for $200$ steps over four seeds, following their settings, and evaluate each trained model on the held-out validation set of all five domains; we report our hyperparameters in \cref{subsec:sdpo_implementation} and full training curves in \cref{app:sdpoothercont}.

\textbf{SDPO improves on the training domain faster, but GRPO converges to comparable accuracies.} In biology, chemistry, and physics, SDPO improves in-domain accuracy faster than GRPO early in training, while GRPO improves gradually over all $200$ steps (\cref{fig:crossvalgrid}).
By step $200$, the two methods are within one standard deviation of each other in biology, material, and physics (SDPO vs.\ GRPO: $52.5$ vs.\ $51.0$, $73.5$ vs.\ $76.2$, and $73.1$ vs.\ $72.4$); the only clear differences are chemistry, where SDPO leads ($77.7$ vs.\ $68.9$), and tool use, where GRPO leads ($53.5$ vs.\ $66.3$) (\cref{tab:cross_domain}). Averaged over the five domains, both methods reach similar in-domain accuracy ($66.1$ for SDPO and $67.0$ for GRPO, up from $49.6$ for the base model).

\textbf{SDPO degrades out of domain, while GRPO does not.} Training SDPO on any of the four science domains lowers average out-of-domain accuracy by $7.1$ to $16.6$ points, whereas GRPO slightly improves it ($+0.9$ to $+1.6$; \cref{tab:cross_domain}). Training on tool use is the exception, with SDPO slightly improving out-of-domain accuracy ($+0.9$), with GRPO slightly lowering it ($-2.1$). The drops are largest when evaluating on material, physics, and tool use, where SDPO's accuracy peaks early in training and then declines steadily (\cref{fig:crossvalgrid}). SDPO's in-domain gains thus come at the cost of out-of-domain performance that policy-gradient training generally preserves.

\subsection{OPSD gains depend on thinking mode transfer, not privileged information} \label{subsec:opsd}
OPSD \citep{opsd} reports gains on AIME24, AIME25, and HMMT25 \citep{aime,balunovic2025matharena} by distilling a thinking-on Qwen3 teacher conditioned on a reference solution $z$ into the same model with thinking off (\cref{app:qwenthink}). The teacher thus differs from the student in two ways: it sees $z$, and it thinks. %

We therefore seek to isolate the influence of the content of $z$ with the teacher thinking mode.  If OPSD gains come from conditioning on $z$, replacing it with the reference solution of a different problem or removing it entirely should degrade performance.

\textbf{Setup.} Following \citet{opsd}, we train on OpenThoughts \citep{openthoughts} math problems, each paired with a reference solution. We fix the student to thinking-off and ablate the teacher's thinking mode (on or off) and its privileged information: the problem's own reference solution (ref $z$), no reference solution (no $z$), or the reference solution of a different problem (shuffled $z$). The teacher template is otherwise fixed, so the three differ only in the content of $z$.

\citet{opsd} train Qwen3-1.7B to 8B with OPSD for 100 steps and compare against thinking-on GRPO (500 steps) and SFT (100 steps). To evaluate fully converged performance, we train OPSD and SFT for 200 steps and GRPO, with thinking on and off, for 500 steps, and evaluate on the same datasets with thinking on and off. As \citet{sdft,sdpo} suggest self-distillation works better for larger models, we also test Qwen3-14B and a second model family, Olmo-3-7B-Instruct. As the latter has no thinking toggle, we instead compare two teachers given the same privileged information: the student itself (self-distillation) and Olmo-3-7B-Think, a separately trained thinking model from the same family (cross-model distillation), which cleanly separates privileged information from thinking-mode transfer. We report Qwen3-8B validation results in \cref{tab:opsd8b}, remaining results and statistics in \cref{app:opsdcont}, and hyperparameters in \cref{tab:opsdhyperparams}.

\begin{table}[h]
  \caption{Qwen3-8B training methods evaluated under \texttt{avg@12} with thinking-on and thinking-off evaluation. OPSD think/no think denote whether teacher thinking was on during OPSD training. The subscript on GRPO runs denotes whether thinking was on or off during training. Denote $z$ as the reference solution shown to the teacher. For every evaluation mode, we report the checkpoint with the best average performance (averaged over four seeds), reporting mean$\pm$std accuracy.}
  \label{tab:opsd8b}
  \centering
  \setlength{\tabcolsep}{4pt}
  \begin{adjustbox}{center,max width=0.9\textwidth}
  \begin{tabular}{lcccccccc}
    \toprule
     & \multicolumn{4}{c}{\textbf{Thinking-on eval}} & \multicolumn{4}{c}{\textbf{Thinking-off eval}} \\
    \cmidrule(lr){2-5} \cmidrule(lr){6-9}
    \textbf{Method} & \textbf{AIME24} & \textbf{AIME25} & \textbf{HMMT25} & \textbf{Avg.} & \textbf{AIME24} & \textbf{AIME25} & \textbf{HMMT25} & \textbf{Avg.} \\
    \midrule
    \multicolumn{9}{l}{\textit{Baselines}} \\
    \quad Base & 75.6 & 68.6 & 45.0 & 63.1 & 26.9 & 21.4 & 9.4 & 19.3 \\
    \quad GRPO$_{\text{think}}$ & 76.5$\pm$0.3 & 69.4$\pm$1.0 & 45.1$\pm$2.7 & 63.7 & 28.5$\pm$0.8 & 22.4$\pm$1.7 & 11.8$\pm$1.5 & 20.9 \\
    \quad GRPO$_{\text{nothink}}$ & 76.5$\pm$0.5 & 70.8$\pm$1.4 & 45.6$\pm$1.8 & 64.3 & 34.6$\pm$2.0 & 23.6$\pm$1.3 & 12.9$\pm$1.1 & 23.7 \\
    \quad SFT  & 76.0$\pm$1.7 & 70.0$\pm$0.5 & 43.0$\pm$1.7 & 63.0 & \textbf{65.1$\pm$1.0} & \textbf{51.2$\pm$2.0} & \textbf{33.8$\pm$1.0} & \textbf{50.0} \\
    \midrule
    \multicolumn{9}{l}{\textit{OPSD (teacher thinking-on)}} \\
    \quad Ref $z$  & 77.2$\pm$1.1 & 71.6$\pm$1.1 & 46.4$\pm$0.4 & 65.0 & 45.1$\pm$1.6 & 34.5$\pm$0.7 & 19.8$\pm$2.4 & 33.1 \\
    \quad No $z$       & 78.1$\pm$1.0 & 71.0$\pm$1.4 & 46.7$\pm$1.3 & \textbf{65.3} & 57.6$\pm$1.2 & 45.6$\pm$2.0 & 30.6$\pm$2.9 & 44.6 \\
    \quad Shuffled $z$ & \textbf{78.3$\pm$0.6} & 71.2$\pm$0.9 & 45.9$\pm$1.0 & 65.1 & 41.5$\pm$7.3 & 36.7$\pm$4.5 & 21.2$\pm$3.4 & 33.2 \\
    \midrule
    \multicolumn{9}{l}{\textit{OPSD (teacher thinking-off)}} \\
    \quad Ref $z$ & 77.2$\pm$0.8 & 70.2$\pm$1.7 & 46.2$\pm$1.3 & 64.6 & 26.8$\pm$1.6 & 22.7$\pm$0.8 & 9.1$\pm$0.5 & 19.5 \\
    \quad No $z$       & 76.4$\pm$1.4 & \textbf{72.1$\pm$1.3} & 44.9$\pm$0.3 & 64.5 & 26.5$\pm$2.3 & 22.6$\pm$0.7 & 12.1$\pm$0.4 & 20.4 \\
    \quad Shuffled $z$ & 77.6$\pm$0.6 & 70.2$\pm$0.8 & \textbf{46.8$\pm$0.9} & 64.9 & 26.9$\pm$1.9 & 19.2$\pm$1.5 & 11.2$\pm$1.7 & 19.1 \\
    \bottomrule
  \end{tabular}
  \end{adjustbox}
\end{table}

\textbf{The content of $z$ does not matter.} Holding the teacher's thinking mode and the evaluation mode fixed, we ablate $z$; on a sample rollout, removing it barely changes the teacher's token probabilities (\cref{fig:teaser}). ``No $z$'' often matches or outperforms ``ref $z$'', especially in \citet{opsd}'s original setting (thinking teachers and thinking-on evaluations), and ``ref $z$'' achieves the best average accuracy in only three of $18$ configurations, each by under one point. In the remaining thinking-on evaluations, the best choice is always ``no $z$'' or ``shuffled $z$''. Except for thinking-off evaluations with a thinking-on teacher, the average accuracies of the three choices of $z$ lie within $3.6$ points of each other in every configuration other than Qwen3-1.7B with a non-thinking teacher (up to $5.2$; \cref{tab:opsd8b}, \cref{app:opsdcont}). With a thinking-on teacher, thinking-off accuracy varies more across $z$ (up to $11.7$ points), but ``ref $z$'' is never the best choice: ``no $z$'' or ``shuffled $z$'' beats it by up to $11.5$ points on Qwen3-8B. \textit{At every scale, a variant lacking either the presence or the relevance of a correct reference matches ``ref $z$'', so neither is required for OPSD's gains.}

\textbf{Thinking-on gains are small and do not discriminate between hypotheses.} OPSD's headline results are thinking-on evaluations of students trained with thinking teachers, which \citet{opsd} attribute to teacher reasoning transfer. Beyond the smallest model, gains are modest ($7.4$, $0.8$, $2.2$, and $1.1$ points for Qwen3-1.7B to Qwen3-14B for the best $z$), and every variant tested recovers them: thinking-off distillation ($6.1$, $0.8$, $1.8$, and $0.7$), any choice of $z$, and, on Qwen3-8B, a loss restricted to the last $10\%$ of tokens (\cref{app:opsdmasking}). GRPO trained thinking-off is also always slightly larger than thinking-on GRPO. Thinking-on gains are thus not evidence for any particular transfer mechanism, so we focus on thinking-off evaluations, the student's native inference mode.

\textbf{Thinking-on teachers induce student thinking.} Rather than improving student direct-answer capabilities, OPSD with a thinking teacher relies on thinking-mode transfer. Enabling teacher thinking improves average thinking-off accuracy by $7.6$, $22.9$, $17.3$, and $28.9$ points for Qwen3-1.7B to Qwen3-14B, whereas replacing ref $z$ with shuffled $z$ changes it by $2.15$, $-0.2$, $-0.15$, and $-0.15$ points. Likewise, Olmo-3-7B-Instruct gains $3.6$ points with a self-teacher but $13.8$ with Olmo-3-7B-Think, averaged over choices of $z$ (\cref{tab:opsdolmo7b}). Since \citet{opsd} report only results with a thinking teacher, their remaining thinking-off gains are driven by the teacher's thinking mode, not by privileged information. This transfer is visible in the students' outputs: in thinking-off evaluations, Qwen3 students distilled from a thinking teacher often emit an \textit{orphan} \texttt{</think>} token, closing a thinking block that the chat template has already closed (\cref{app:qwenthink}). These rollouts reason in the first person, box an answer, emit \texttt{</think>}, then write a clean solution, mirroring a thinking-on response (\cref{fig:orphantokenex,fig:base8bthinkoff,fig:base8bthinkon}). Among on-policy methods, orphan tokens occur mainly in thinking-off rollouts from OPSD with thinking teachers (up to $69.9\%$); GRPO and the base model produce none, and OPSD with a thinking-off teacher at most $6.7\%$. Thinking-on teachers also increase response length by $2.72$ to $5.82$ times, approaching the base model's thinking-on lengths (\cref{tab:qwen38b_summary}, \cref{app:opsdcont}), whereas length stays more stable with teacher thinking disabled ($0.8$ to $2.52$ times the base model's). Thus, with teacher thinking enabled, OPSD's thinking-off gains come from forcing thinking-mode behavior into thinking-off mode, which requires no privileged information.

\subsection{Self-distillation gradients do not align with policy gradients}\label{subsec:orthogonality}
\begin{wraptable}{R}{0.42\textwidth}
  \vspace{-\intextsep}
  \centering
  \footnotesize
  \setlength{\tabcolsep}{3pt}
  \caption{Cosine similarity with the held-out GRPO$_{\text{nothink}}$ gradient on math for Qwen3-8B (pooled over four seeds $\pm$ jackknife error).}
  \label{tab:gradcos_math}
  \begin{adjustbox}{max width=\linewidth}
  \begin{tabular}{lc}
    \toprule
    \textbf{Method} & \textbf{Similarity} \\
    \midrule
    \multicolumn{2}{l}{\textit{Baselines}} \\
    \quad GRPO$_{\text{think}}$ & $-0.014\pm0.019$ \\
    \quad GRPO$_{\text{nothink}}$ & $\mathbf{+0.118\pm0.050}$ \\
    \quad SFT & $-0.075\pm0.031$ \\
    \midrule
    \multicolumn{2}{l}{\textit{OPSD (teacher thinking-on)}} \\
    \quad Ref $z$ & $-0.160\pm0.020$ \\
    \quad No $z$ & $-0.157\pm0.019$ \\
    \quad Shuffled $z$ & $-0.157\pm0.019$ \\
    \midrule
    \multicolumn{2}{l}{\textit{OPSD (teacher thinking-off)}} \\
    \quad Ref $z$ & $-0.163\pm0.013$ \\
    \quad No $z$ & $-0.172\pm0.018$ \\
    \quad Shuffled $z$ & $-0.163\pm0.011$ \\
    \bottomrule
  \end{tabular}
  \end{adjustbox}
  \vspace{1.5em} %
  \centering
  \footnotesize
  \setlength{\tabcolsep}{3pt}
  \caption{Cosine similarity with the held-out GRPO gradient on LiveCodeBench v6 (\textsc{test\_unseen}) for methods trained on \textsc{train\_seen} (pooled over four seeds $\pm$ jackknife error).}
  \label{tab:gradcos_lcb}
  \begin{adjustbox}{max width=\linewidth}
  \begin{tabular}{lcc}
    \toprule
    \textbf{Model} & \textbf{GRPO} & \textbf{SDPO} \\
    \midrule
    Qwen3-8B & $\mathbf{+0.121\pm0.011}$ & $-0.167\pm0.063$ \\
    Qwen3-14B & $\mathbf{+0.070\pm0.042}$ & $+0.013\pm0.112$ \\
    Olmo-3-7B-Instruct & $\mathbf{+0.422\pm0.091}$ & $+0.062\pm0.028$ \\
    \bottomrule
  \end{tabular}
  \end{adjustbox}
\end{wraptable}

The preceding sections examine the behavior of trained models. We now examine the self-distillation update itself and ask whether a self-distillation step moves the policy in a direction that increases reward on held-out problems. Let $J_{\text{held}}(\theta)$ be $\pi_\theta$'s expected reward on held-out problems and $g$ be a given method's update direction. To a first order approximation,
\begin{equation}
J_{\text{held}}(\theta_0 + \eta g) - J_{\text{held}}(\theta_0) = \eta \langle \nabla_\theta J_{\text{held}}(\theta_0), g\rangle + O(\eta^2),
\end{equation}
so for $\eta$ sufficiently small, a step along $g$ increases held-out reward if and only if $g$ has positive cosine similarity with $\nabla_\theta J_{\text{held}}$. We estimate $\nabla_\theta J_{\text{held}}$ by the GRPO gradient on held-out problems and $g$ by the self-distillation gradient and propose that their cosine similarity therefore measures whether a method's update generalizes locally, without needing to train to convergence. This cosine similarity metric is also loosely motivated by \cref{subsec:equivalence}): if the self-conditioned teacher approximates $\pi^*$ (\cref{eq:prop3.1}), as \citet{sdft} assume, the distillation gradient should follow the policy gradient on the training problems. Here we ask the analogous question on held-out problems, where the policy gradient is the direction that improves generalization.

\textbf{Setup.}
Practical self-distillation losses differ from the idealized objectives of \cref{prop:equivalence}: they use other divergences such as Jensen--Shannon, per-token clipping, top-$k$ logit approximations, and EMA teachers (\cref{eq:sd_estimator}), so rather than testing the proposition's equality directly, we test a consequence: a method whose teacher approximates the optimal teacher should update the model in the same direction as the policy gradient. We therefore measure the cosine similarity between each method's training gradient and a \textit{reference}, the GRPO gradient on held-out problems, which is the policy gradient direction that increases reward on new problems. To a first order approximation, a small step along a method's gradient increases held-out reward if its cosine similarity with the reference is positive, and decreases it if negative.
All gradients are estimated at the base model: on-policy RL favors high-reward policies that stay close to the base model in KL \citep{razor}, so we test whether the teacher is KL-optimal relative to the base model. We use the same setup as \cref{subsec:sdpo} and \cref{subsec:opsd}: for math, methods train on OpenThoughts and the reference is GRPO$_{\text{nothink}}$ on AIME24, AIME25, and HMMT25 (held-out math policy gradient); for LiveCodeBench v6, methods train on \textsc{train\_seen} and the reference is GRPO on \textsc{test\_unseen}. We report our complete results in \cref{app:gradprobecont}. Since cosine similarity is a nonlinear function of the pooled gradients and we estimate gradients over four seeds, we estimate its error with a leave-one-seed-out jackknife; we detail our gradient estimation method in \cref{app:gradprobe}.

\textbf{OPSD points away from the held-out math policy gradient.} Across Qwen3-1.7B to 14B, GRPO$_{\text{nothink}}$ is the only method with positive cosine similarity with the held-out policy gradient on every model ($+0.018$ to $+0.166$), and clearly so on Qwen3-8B and 14B ($+0.118\pm0.050$ and $+0.166\pm0.029$; \cref{tab:gradcos_math,tab:gradcos_opsd_all}). In contrast, all $24$ OPSD variants have negative cosine similarity ($-0.205$ to $-0.067$). On Qwen3-8B, the six variants lie between $-0.172$ and $-0.157$, within each other's error bars. Notably, OPSD with a thinking-on teacher points away from the held-out policy gradient despite its large thinking-off gains (\cref{tab:opsd8b}): rather than moving the student towards its own high-reward solutions like policy gradients, it improves performance by matching the teacher's thinking style (\cref{subsec:opsd}). Instead, on Olmo-3-7B-Instruct, GRPO and all six OPSD variants have small positive cosine similarity ($+0.026$ to $+0.062$), but their jackknife errors ($0.021$ to $0.054$) are comparable to the values themselves, so none are clearly resolved from zero (\cref{tab:gradcos_opsd_olmo}).

\textbf{SDPO gradients are less aligned than GRPO gradients.} On LiveCodeBench v6, SDPO is less aligned with the held-out policy gradient than GRPO on all three models, and points away from it on Qwen3-8B (\cref{tab:gradcos_lcb}). These results are consistent with SDPO's failure to improve \textsc{test\_unseen} accuracy (\cref{subsec:sdpo}), suggesting the self-distillation objective is only weakly correlated with generalizable reward maximization.

\textbf{Self-distillation gradients have low variance but are misaligned.} OPSD and SDPO gradients are nearly identical across seeds: their gradient noise scale \citep{gradientnoisescale}, a normalized measure of variance across seeds where $0$ is noiseless (\cref{app:gradprobe}), is at most $0.05$, versus $1.97$ to $4.11$ for GRPO across all models (see \cref{app:gradprobecont}). Self-distillation thus delivers the low-variance signal it promises, but that signal does not point along the held-out policy gradient: the self-conditioned teacher does not approximate the optimal teacher of \cref{prop:equivalence}.

\section{Limitations}\label{sec:limitations}
Our experiments are limited to the Qwen3 and Olmo model families and the SDPO and OPSD self-distillation recipes. We do not address the more recent developments on self-distillation techniques, which include harnesses \citep{ophsd}, contrastive formulations \citep{rlcsd}, or more complex aggregation of feedback from neighboring rollouts \citep{multirollout}, which may combat the limitations of these self-distillation methods. Moreover, we focus on verifiable domains and acknowledge that self-distillation has use cases in non-verifiable settings \citep{aligning}.

\section{Conclusion}
We show that the gains attributed to privileged self-distillation either do not generalize or do not come from privileged information. SDPO's improvements on LiveCodeBench v6 are confined to problems seen during training, and its gains in other domains are come at the cost of out-of-distribution performance. OPSD's math gains rely on thinking-mode transfer, with privileged information having little effect. Finally, self-distillation gradients do not align with the policy gradient on held-out problems, indicating that the self-conditioned teacher does not approximate the optimal teacher. Together, these results highlight the need to develop principled reinforcement learning methods that can exploit privileged information.

\subsection*{AI use statement}
In this work, we used generative AI tools to propose or refine hypotheses, design or provide feedback on research methodology or experiments, implement methods, support qualitative and thematic data analysis, and interpret results. We did not use generative AI tools to help develop theoretical models or conceptual frameworks, formulate mathematical claims, provide critical ingredients for proving mathematical claims, or assist in writing proofs. Generating synthetic datasets, translation, and cleaning or reformatting datasets are not applicable to this work. Additionally, we used generative AI tools to create or modify scientific figures, summarize or analyze existing literature, discover research topics or identify gaps, brainstorm, search for information, identify relevant literature, propose a title or keywords, draft parts of the paper, edit the paper for readability, and create or edit software code. We have reviewed all AI-assisted work: two authors reviewed and tested LLM-generated code for correctness, and we checked all LLM-retrieved examples. We take responsibility for the final content of this work, including text, claims, and artifacts produced with the aid of generative AI.

\subsection*{Reproducibility statement}
We make our code publicly available at \coderepo. We detail our experimental setups and hyperparameters for SDPO and GRPO on LiveCodeBench v6, SciKnowEval, and ToolAlpaca in \cref{subsec:sdpo_implementation}; for OPSD, GRPO, and SFT on OpenThoughts in \cref{app:opsdhyperparams}; for the gradient similarity analysis in \cref{app:gradprobe}; and the OPSD prompts in \cref{app:opsdprompts}. The proof of \cref{prop:equivalence} is in \cref{app:propproof}.

\bibliography{iclr2027_conference}
\bibliographystyle{iclr2027_conference}

\newpage
\appendix
\crefalias{section}{appendix}
\crefalias{subsection}{subappendix}
\crefalias{subsubsection}{subsubappendix}

 \section{Extended related works} \label{app:extended_related_works}
We divide existing work on the limitations of self-distillation techniques into the following:

\textbf{Empirical re-evaluations.} \citet{sdpobadoodifbadteachersignal} demonstrate that SDPO struggles to generalize to out-of-distribution domains, and fail in continual learning settings where GRPO succeeds. \citet{sdthinkbad} highlight that self-distillation degrades thinking-models; with previous works primarily focusing on instruction-tuned models. Concurrent to our work, \citet{opsdscoop,opsdother,ichihara2026privileged} show that OPSD does not rely on privileged information, but on thinking-mode transfer. In contrast to these works, we provide our results with an \textit{exact} reproduction of SDPO and OPSD training and evaluating settings, and we ask whether SDPO's gains in other domains reflect improvements in reasoning.

\textbf{Token-level symptoms.} \citet{kim2026degrade,sdthinkbad,rlcsd,microsoft} show that self-distillation losses concentrate on tokens which express uncertainty and have high entropy. This results in a suppression of the student's backtracking behavior and collapse in response lengths. Instead, we focus on the positive results reported by OPSD, and find through token-level diagnoses and response length metrics, that performance improvements relate to thinking-mode transfer.

\textbf{Sequence-level diagnoses.} \citet{pinotvalue} develop statistics that indicate that the self-conditioned teacher cannot differentiate between correct and incorrect solutions generated by the student. \citet{microsoft} present analogous results, showing that the self-conditioned teacher is greatly biased towards the reference solution it is provided, mirroring \citet{yang2026selfdistilledrlvr}'s results indicating that the teacher leaks privileged information. On the other hand, we diagnose whether the self-conditioned teacher approximates the optimal KL-regularized model, which is an assumption behind prior self-distillation work \citep{sdft}.

\textbf{Remedies.} \citet{rlcsd} and \citet{yang2026selfdistilledrlvr} retain GRPO as a learning signal and use the self-conditioned teacher only to modulate its per-token advantages. Because these methods do not rely on the self-distillation loss alone, which is the setting we study, they fall outside the scope of this work. Their reliance on GRPO is consistent with \citet{microsoft}, who question whether the self-distillation objective on its own, without such auxiliary machinery, yields any performance gains at all.

\section{Proof of \cref{prop:equivalence}} \label{app:propproof}
We restate the proposition for convenience.
\equivprop*

\begin{proof}
We use the notation of the proposition, suppressing $x$. Since $\mathcal{Y}$ is finite, all expectations over $\mathcal{Y}$ are finite sums, so gradients commute with expectations, and we have the log-derivative identity $\nabla_\theta \pi_\theta(y) = \pi_\theta(y) \nabla_\theta \log \pi_\theta(y)$ and the score identity $\mathbb{E}_{y \sim \pi_\theta}[\nabla_\theta \log \pi_\theta(y)] = \sum_{y} \nabla_\theta \pi_\theta(y) = 0$. Let $J(\theta) = \mathbb{E}_{y \sim \pi_\theta}[R(y)]$ and $g = \nabla_\theta J(\theta)$. We show that both KL gradients are negative multiples of $g$.

\textbf{Policy gradient.} By the log-derivative identity,
\begin{equation*}
    g = \mathbb{E}_{y \sim \pi_\theta}\left[ R(y)\, \nabla_\theta \log \pi_\theta(y) \right].
\end{equation*}

\textbf{Reverse KL.} Since $\sg{(\pi^*)}$ is constant in $\theta$,
\begin{align*}
    \nabla_\theta \KL{\pi_\theta}{\sg{(\pi^*)}}
    &= \sum_{y} (\nabla_\theta \pi_\theta(y)) \log \frac{\pi_\theta(y)}{\pi^*(y)} + \sum_{y} \pi_\theta(y)\, \nabla_\theta \log \pi_\theta(y) \\
    &= \mathbb{E}_{y \sim \pi_\theta}\left[ \log \frac{\pi_\theta(y)}{\pi^*(y)}\, \nabla_\theta \log \pi_\theta(y) \right] && \text{(score identity)} \\
    &= \mathbb{E}_{y \sim \pi_\theta}\left[ \left( \log Z - R(y)/\beta \right) \nabla_\theta \log \pi_\theta(y) \right] && \text{(definition of $\pi^*$)} \\
    &= -\frac{1}{\beta}\, g. && \text{(score identity)}
\end{align*}

\textbf{Forward KL.} Since $\pi^*(y) > 0$ if and only if $\pi_\theta(y) > 0$, we may change measure from $\pi^*$ to $\pi_\theta$:
\begin{align*}
    \nabla_\theta \KL{\sg{(\pi^*)}}{\pi_\theta}
    &= -\mathbb{E}_{y \sim \pi^*}\left[ \nabla_\theta \log \pi_\theta(y) \right] \\
    &= -\mathbb{E}_{y \sim \pi_\theta}\left[ \frac{\pi^*(y)}{\pi_\theta(y)}\, \nabla_\theta \log \pi_\theta(y) \right] \\
    &= -\frac{1}{Z}\, \mathbb{E}_{y \sim \pi_\theta}\left[ \exp(R(y)/\beta)\, \nabla_\theta \log \pi_\theta(y) \right].
\end{align*}
Because $R(y) \in \{0,1\}$, we have $\exp(R(y)/\beta) = 1 + (e^{1/\beta} - 1) R(y)$. Substituting and applying the score identity to the constant term,
\begin{equation*}
    \nabla_\theta \KL{\sg{(\pi^*)}}{\pi_\theta} = -\frac{e^{1/\beta} - 1}{Z}\, g.
\end{equation*}

\textbf{Conclusion.} Both $1/\beta$ and $(e^{1/\beta} - 1)/Z$ are strictly positive for $\beta > 0$. Hence, whenever $g \neq \vec{0}$, the negated gradients of both KL objectives are positive multiples of $g$, and normalizing each to unit $\ell_2$ norm yields \cref{eq:prop3.1}.
\end{proof}

\section{Qwen3 thinking modes} \label{app:qwenthink}
Qwen3 models support a thinking and a non-thinking mode in the \textit{same model} \citep{qwen3technicalreport}. For the flagship models (Qwen3-32B and Qwen3-235B-A22B), the two modes are merged during post-training (\textit{Thinking Mode Fusion}) by fine-tuning on a mixture of ``thinking'' and ``non-thinking'' data. The smaller models (Qwen3 0.6B to 14B and Qwen3-30B-A3B) inherit both of these modes through supervised fine-tuning and on-policy-distillation from the flagship models. Concretely, in thinking mode, the model reasons inside a \texttt{<think>}\ldots\texttt{</think>} block before responding. The mode is toggled via the \texttt{enable\_thinking} argument of the chat template. When it is set to False, an empty think block is prefilled in the assistant turn:
\begin{verbatim}
<|im_start|>assistant
<think>

</think>
\end{verbatim}
In non-thinking mode, the model is trained to respond directly after this empty thinking block.

\section{Extended experimental results}

\subsection{SDPO LiveCodeBench v6 performance}\label{app:sdpolcbcont}
\Cref{fig:sdpolcb14b,fig:sdpolcbolmo} show the LiveCodeBench v6 training curves for Qwen3-14B and Olmo-3-7B-Instruct, complementing \cref{fig:sdpolcb} for Qwen3-8B. \Cref{tab:sdpolcb_trainseen,tab:sdpolcb_trainunseen,tab:sdpolcb_testunseen} also reports the scores of all three models on \textsc{test\_unseen}, \textsc{train\_seen}, and \textsc{train\_unseen} at training steps 0, 80, and 200.

\begin{figure}[h]
\begin{center}
\includegraphics[width=0.9\textwidth]{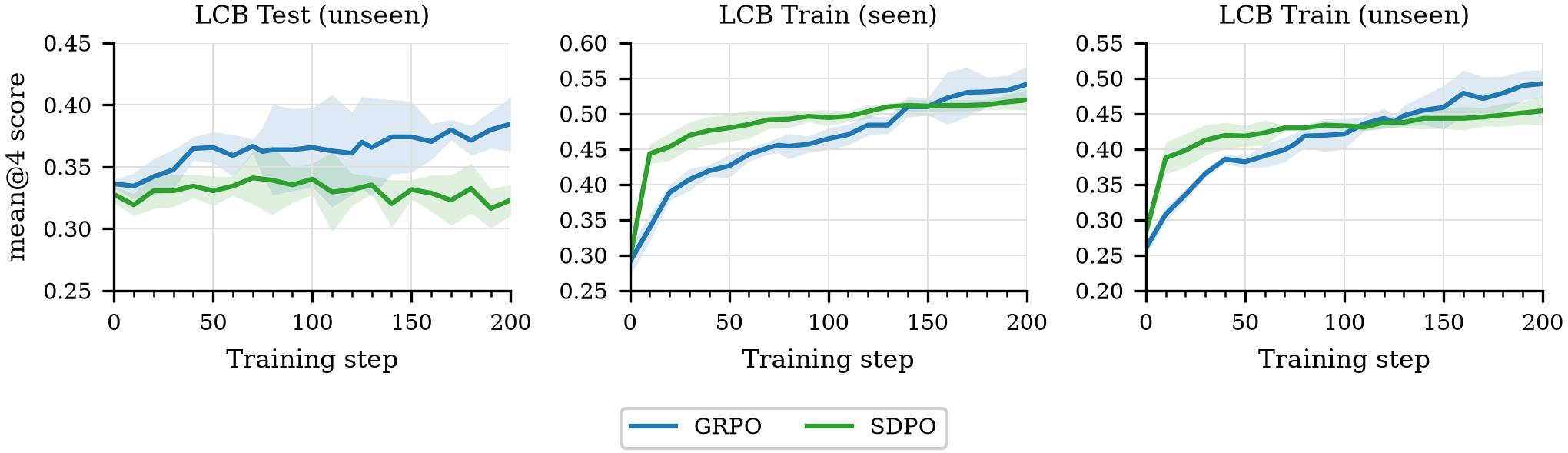}
\end{center}
\caption{Qwen3-14B SDPO vs.\ GRPO on LiveCodeBench v6 across four seeds, scored as the fraction of problems with no failing test cases (mean $\pm$ std). \textsc{train\_seen} and \textsc{train\_unseen} evaluate problems seen during training, on seen and on held-out test cases respectively; \textsc{test\_unseen} evaluates held-out problems.}
\label{fig:sdpolcb14b}
\end{figure}

\begin{figure}[h]
\begin{center}
\includegraphics[width=0.9\textwidth]{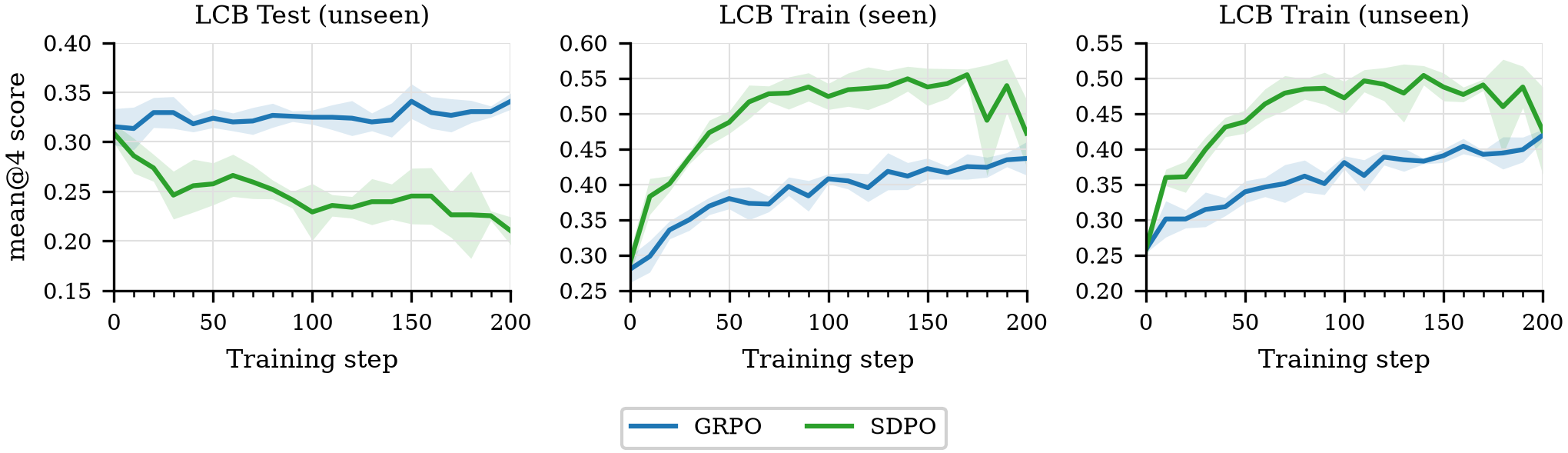}
\end{center}
\caption{Olmo-3-7B-Instruct SDPO vs.\ GRPO on LiveCodeBench v6 across four seeds, scored as the fraction of problems with no failing test cases (mean $\pm$ std). \textsc{train\_seen} and \textsc{train\_unseen} evaluate problems seen during training, on seen and on held-out test cases respectively; \textsc{test\_unseen} evaluates held-out problems.}
\label{fig:sdpolcbolmo}
\end{figure}

\begin{table}[h]
  \caption{LiveCodeBench v6 score (fraction of problems with no failing test cases, mean$\pm$std over four seeds) on \textsc{test\_unseen} for SDPO and GRPO at initialization, 80 steps, and 200 steps. Bold marks the better method at each step.}
  \label{tab:sdpolcb_testunseen}
  \centering
  \setlength{\tabcolsep}{4pt}
  \begin{tabular}{llccc}
    \toprule
    \textbf{Model} & \textbf{Method} & \textbf{Step 0} & \textbf{Step 80} & \textbf{Step 200} \\
    \midrule
    Qwen3-8B & SDPO & \textbf{32.3$\pm$1.6} & 30.1$\pm$0.7 & 29.8$\pm$1.1 \\
     & GRPO & 30.6$\pm$1.0 & \textbf{33.3$\pm$1.9} & \textbf{34.3$\pm$0.5} \\
    \midrule
    Qwen3-14B & SDPO & 32.8$\pm$0.7 & 33.9$\pm$2.8 & 32.3$\pm$1.3 \\
     & GRPO & \textbf{33.6$\pm$0.4} & \textbf{36.4$\pm$3.7} & \textbf{38.5$\pm$2.2} \\
    \midrule
    Olmo-3-7B-Instruct & SDPO & 30.9$\pm$1.0 & 25.2$\pm$1.0 & 21.0$\pm$1.4 \\
     & GRPO & \textbf{31.5$\pm$1.8} & \textbf{32.7$\pm$1.2} & \textbf{34.1$\pm$0.8} \\
    \bottomrule
  \end{tabular}
\end{table}
\begin{table}[h]
  \caption{LiveCodeBench v6 score (fraction of problems with no failing test cases, mean$\pm$std over four seeds) on \textsc{train\_seen} for SDPO and GRPO at initialization, 80 steps, and 200 steps. Bold marks the better method at each step.}
  \label{tab:sdpolcb_trainseen}
  \centering
  \setlength{\tabcolsep}{4pt}
  \begin{tabular}{llccc}
    \toprule
    \textbf{Model} & \textbf{Method} & \textbf{Step 0} & \textbf{Step 80} & \textbf{Step 200} \\
    \midrule
    Qwen3-8B & SDPO & \textbf{25.8$\pm$1.8} & \textbf{48.6$\pm$0.7} & 48.9$\pm$1.6 \\
     & GRPO & 25.1$\pm$0.7 & 43.8$\pm$2.4 & \textbf{50.9$\pm$1.1} \\
    \midrule
    Qwen3-14B & SDPO & \textbf{29.4$\pm$0.4} & \textbf{49.2$\pm$1.3} & 51.9$\pm$1.5 \\
     & GRPO & 29.0$\pm$2.0 & 45.4$\pm$1.8 & \textbf{54.1$\pm$2.5} \\
    \midrule
    Olmo-3-7B-Instruct & SDPO & \textbf{28.9$\pm$1.7} & \textbf{52.9$\pm$2.3} & \textbf{47.2$\pm$4.7} \\
     & GRPO & 28.0$\pm$1.9 & 39.7$\pm$1.3 & 43.7$\pm$2.4 \\
    \bottomrule
  \end{tabular}
\end{table}

\begin{table}[h]
  \caption{LiveCodeBench v6 score (fraction of problems with no failing test cases, mean$\pm$std over four seeds) on \textsc{train\_unseen} for SDPO and GRPO at initialization, 80 steps, and 200 steps. Bold marks the better method at each step.}
  \label{tab:sdpolcb_trainunseen}
  \centering
  \setlength{\tabcolsep}{4pt}
  \begin{tabular}{llccc}
    \toprule
    \textbf{Model} & \textbf{Method} & \textbf{Step 0} & \textbf{Step 80} & \textbf{Step 200} \\
    \midrule
    Qwen3-8B & SDPO & \textbf{25.6$\pm$0.5} & \textbf{44.1$\pm$1.0} & 44.0$\pm$1.7 \\
     & GRPO & 25.4$\pm$1.0 & 41.2$\pm$2.4 & \textbf{46.2$\pm$2.0} \\
    \midrule
    Qwen3-14B & SDPO & \textbf{28.2$\pm$1.4} & \textbf{43.0$\pm$0.2} & 45.4$\pm$2.1 \\
     & GRPO & 26.0$\pm$1.0 & 41.8$\pm$1.6 & \textbf{49.2$\pm$2.0} \\
    \midrule
    Olmo-3-7B-Instruct & SDPO & \textbf{25.9$\pm$1.6} & \textbf{48.5$\pm$1.4} & \textbf{42.6$\pm$6.2} \\
     & GRPO & 25.8$\pm$0.4 & 36.2$\pm$2.3 & 41.9$\pm$0.8 \\
    \bottomrule
  \end{tabular}
\end{table}

\subsection{SDPO science reasoning and tool use performance}\label{app:sdpoothercont}
\Cref{fig:crossvalgrid} shows the validation accuracy of GRPO and SDPO on all five SciKnowEval and ToolAlpaca domains throughout training, complementing the step $200$ results in \cref{tab:cross_domain} (\cref{subsec:sft}).

\begin{figure}[h]
  \centering
  \makebox[\textwidth][c]{\includegraphics[width=1.3\textwidth]{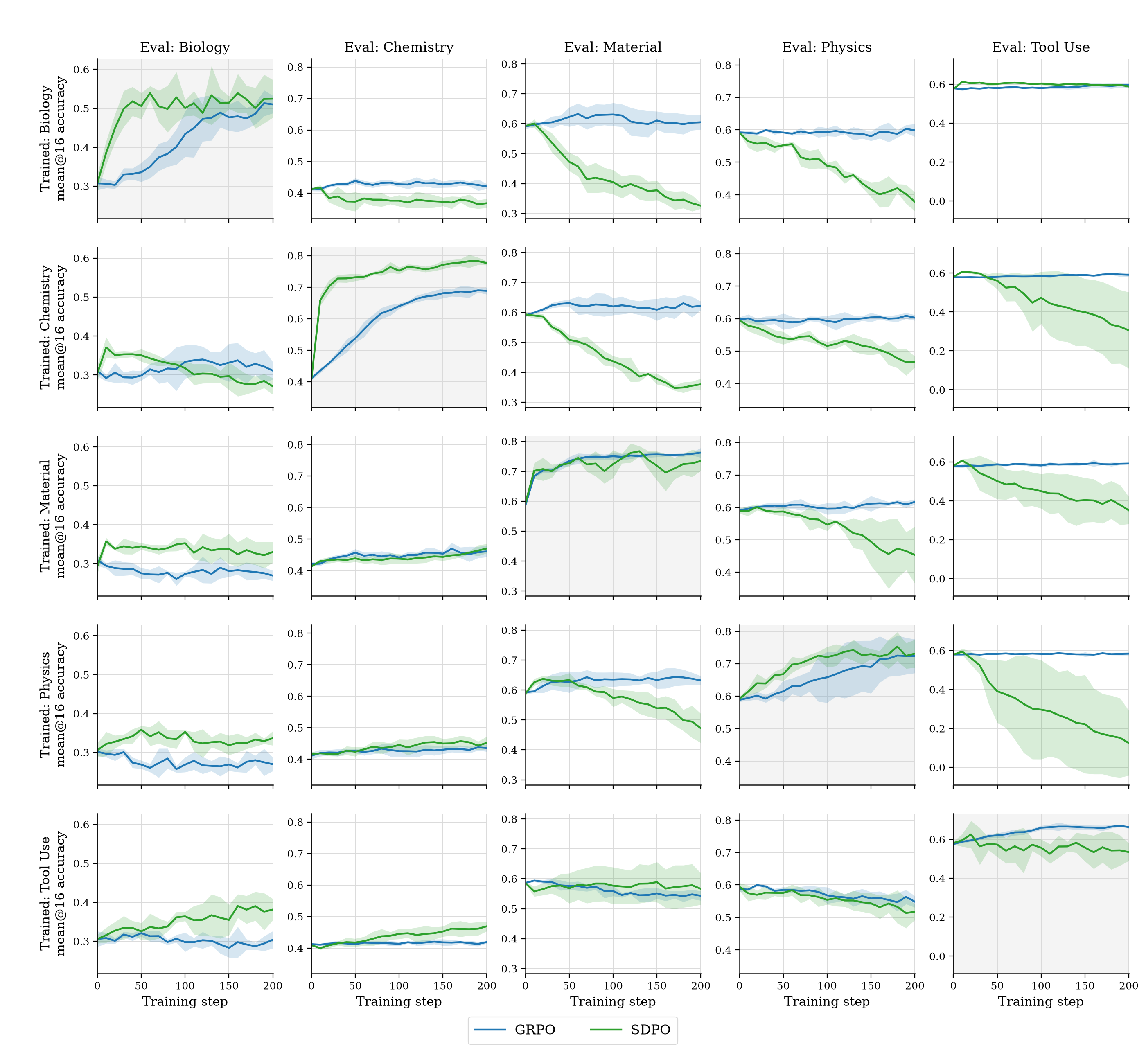}}
  \caption{Cross-domain validation accuracy of Qwen3-8B trained with GRPO and SDPO on SciKnowEval and ToolAlpaca, averaged over four seeds. Panels vary over training domains (across rows) and evaluation domains (across columns).}
  \label{fig:crossvalgrid}
\end{figure}

\FloatBarrier

\subsection{OPSD metrics on Qwen3 models} \label{app:opsdcont}

\begin{table}[h]
  \caption{Qwen3-1.7B training methods evaluated under \texttt{avg@12} with thinking-on and thinking-off evaluation. OPSD think/no think denote whether teacher thinking was on during OPSD training. The subscript on GRPO runs denotes whether thinking was on or off during training. Denote $z$ as the reference solution shown to the teacher. For every evaluation mode, we report the checkpoint with the best average performance (averaged over four seeds), reporting mean$\pm$std accuracy.}
  \label{tab:opsd1.7b}
  \centering
  \setlength{\tabcolsep}{4pt}
  \begin{adjustbox}{center,max width=0.9\textwidth}
  \begin{tabular}{lcccccccc}
    \toprule
     & \multicolumn{4}{c}{\textbf{Thinking-on eval}} & \multicolumn{4}{c}{\textbf{Thinking-off eval}} \\
    \cmidrule(lr){2-5} \cmidrule(lr){6-9}
    \textbf{Method} & \textbf{AIME24} & \textbf{AIME25} & \textbf{HMMT25} & \textbf{Avg.} & \textbf{AIME24} & \textbf{AIME25} & \textbf{HMMT25} & \textbf{Avg.} \\
    \midrule
    \multicolumn{9}{l}{\textit{Baselines}} \\
    \quad Base & 46.7 & 34.4 & 21.9 & 34.4 & 15.8 & 10.0 & 4.4 & 10.1 \\
    \quad GRPO$_\text{think}$ & 49.9$\pm$3.5 & 36.3$\pm$1.3 & 23.5$\pm$3.1 & 36.6 & 15.3$\pm$1.8 & 8.7$\pm$1.9 & 5.9$\pm$1.4 & 10.0 \\
    \quad GRPO$_\text{nothink}$ & 49.9$\pm$2.1 & 37.8$\pm$1.7 & 24.7$\pm$1.1 & 37.4 & 14.7$\pm$1.1 & 11.4$\pm$0.6 & 6.7$\pm$0.6 & 10.9 \\
    \quad SFT  & 50.3$\pm$0.8 & 36.9$\pm$1.3 & 25.8$\pm$1.4 & 37.6 & \textbf{37.8$\pm$3.3} & \textbf{27.6$\pm$1.2} & \textbf{18.3$\pm$2.1} & \textbf{27.9} \\
    \midrule
    \multicolumn{9}{l}{\textit{OPSD (teacher thinking-on)}} \\
    \quad Ref $z$  & \textbf{56.0$\pm$1.2} & \textbf{42.4$\pm$1.6} & 27.1$\pm$0.5 & \textbf{41.8} & 26.7$\pm$2.0 & 19.7$\pm$1.7 & 12.2$\pm$1.3 & 19.5 \\
    \quad No $z$       & 53.0$\pm$2.3 & 39.7$\pm$1.0 & 26.5$\pm$0.9 & 39.7 & 16.4$\pm$1.4 & 13.1$\pm$0.7 & 7.4$\pm$0.9 & 12.3 \\
    \quad Shuffled $z$ & \textbf{56.0$\pm$1.4} & 39.2$\pm$0.7 & \textbf{27.8$\pm$0.5} & 41.0 & 28.8$\pm$1.6 & 21.5$\pm$1.0 & 16.0$\pm$1.2 & 22.1 \\
    \midrule
    \multicolumn{9}{l}{\textit{OPSD (teacher thinking-off)}} \\
    \quad Ref $z$  & 54.1$\pm$1.9 & 39.7$\pm$0.9 & 27.2$\pm$1.2 & 40.3 & 15.2$\pm$0.7 & 11.5$\pm$1.5 & 6.4$\pm$0.7 & 11.0 \\
    \quad No $z$       & 49.0$\pm$2.8 & 36.9$\pm$0.9 & 21.8$\pm$1.3 & 35.9 & 11.2$\pm$0.8 & 7.2$\pm$0.8 & 4.2$\pm$0.7 & 7.5 \\
    \quad Shuffled $z$ & 54.0$\pm$0.8 & 40.1$\pm$2.0 & 27.4$\pm$1.7 & 40.5 & 16.4$\pm$1.6 & 13.2$\pm$1.4 & 8.5$\pm$0.3 & 12.7 \\
    \bottomrule
  \end{tabular}
  \end{adjustbox}
\end{table}

\begin{table}[h]
  \caption{Qwen3-4B training methods evaluated under \texttt{avg@12} with thinking-on and thinking-off evaluation. OPSD think/no think denote whether teacher thinking was on during OPSD training. The subscript on GRPO runs denotes whether thinking was on or off during training. Denote $z$ as the reference solution shown to the teacher. For every evaluation mode, we report the checkpoint with the best average performance (averaged over four seeds), reporting mean$\pm$std accuracy.}
  \label{tab:opsd4b}
  \centering
  \setlength{\tabcolsep}{4pt}
  \begin{adjustbox}{center,max width=0.9\textwidth}
  \begin{tabular}{lcccccccc}
    \toprule
     & \multicolumn{4}{c}{\textbf{Thinking-on eval}} & \multicolumn{4}{c}{\textbf{Thinking-off eval}} \\
    \cmidrule(lr){2-5} \cmidrule(lr){6-9}
    \textbf{Method} & \textbf{AIME24} & \textbf{AIME25} & \textbf{HMMT25} & \textbf{Avg.} & \textbf{AIME24} & \textbf{AIME25} & \textbf{HMMT25} & \textbf{Avg.} \\
    \midrule
    \multicolumn{9}{l}{\textit{Baselines}} \\
    \quad Base & \textbf{76.4} & 67.8 & 43.1 & 62.4 & 21.7 & 20.3 & 11.9 & 18.0 \\
    \quad GRPO$_\text{think}$ & 73.3$\pm$0.9 & 67.0$\pm$1.7 & 44.3$\pm$1.6 & 61.6 & 23.6$\pm$0.9 & 20.8$\pm$1.0 & 11.7$\pm$0.9 & 18.7 \\
    \quad GRPO$_\text{nothink}$ & 74.3$\pm$0.3 & 66.2$\pm$0.8 & 44.9$\pm$0.7 & 61.8 & 28.3$\pm$1.6 & 22.6$\pm$0.8 & 14.4$\pm$0.6 & 21.8 \\
    \quad SFT  & 74.1$\pm$1.6 & 68.3$\pm$1.1 & 42.3$\pm$1.1 & 61.6 & \textbf{61.9$\pm$2.5} & 50.7$\pm$1.4 & \textbf{32.9$\pm$1.5} & \textbf{48.5} \\
    \midrule
    \multicolumn{9}{l}{\textit{OPSD (teacher thinking-on)}} \\
    \quad Ref $z$  & 76.1$\pm$1.0 & 67.2$\pm$0.9 & 44.9$\pm$1.1 & 62.7 & 54.4$\pm$2.1 & \textbf{51.9$\pm$2.2} & 31.1$\pm$2.0 & 45.8 \\
    \quad No $z$       & 74.5$\pm$2.0 & \textbf{70.0$\pm$2.2} & 45.0$\pm$0.8 & \textbf{63.2} & 42.0$\pm$0.7 & 37.3$\pm$2.0 & 24.0$\pm$1.1 & 34.4 \\
    \quad Shuffled $z$ & 75.6$\pm$1.9 & 67.2$\pm$0.6 & \textbf{46.1$\pm$1.4} & 62.9 & 55.9$\pm$2.6 & 51.5$\pm$3.5 & 31.0$\pm$0.8 & 46.1 \\
    \midrule
    \multicolumn{9}{l}{\textit{OPSD (teacher thinking-off)}} \\
    \quad Ref $z$  & 76.2$\pm$0.5 & 66.9$\pm$1.3 & 43.6$\pm$1.5 & 62.2 & 27.3$\pm$1.7 & 21.6$\pm$0.8 & 12.1$\pm$1.0 & 20.3 \\
    \quad No $z$       & 74.0$\pm$0.3 & 66.5$\pm$1.8 & 43.1$\pm$1.1 & 61.2 & 22.6$\pm$0.9 & 19.3$\pm$1.7 & 11.3$\pm$0.1 & 17.8 \\
    \quad Shuffled $z$ & 76.2$\pm$1.1 & 68.1$\pm$2.2 & 45.2$\pm$0.7 & \textbf{63.2} & 26.2$\pm$1.8 & 22.6$\pm$1.6 & 10.1$\pm$0.7 & 19.6 \\
    \bottomrule
  \end{tabular}
  \end{adjustbox}
\end{table}

\begin{table}[h]
  \caption{Qwen3-14B training methods evaluated under \texttt{avg@12} with thinking-on and thinking-off evaluation. OPSD think/no think denote whether teacher thinking was on during OPSD training. The subscript on GRPO runs denotes whether thinking was on or off during training. Denote $z$ as the reference solution shown to the teacher. For every evaluation mode, we report the checkpoint with the best average performance (averaged over four seeds), reporting mean$\pm$std accuracy.}
  \label{tab:opsd14b}
  \centering
  \setlength{\tabcolsep}{4pt}
  \begin{adjustbox}{center,max width=0.9\textwidth}
  \begin{tabular}{lcccccccc}
    \toprule
     & \multicolumn{4}{c}{\textbf{Thinking-on eval}} & \multicolumn{4}{c}{\textbf{Thinking-off eval}} \\
    \cmidrule(lr){2-5} \cmidrule(lr){6-9}
    \textbf{Method} & \textbf{AIME24} & \textbf{AIME25} & \textbf{HMMT25} & \textbf{Avg.} & \textbf{AIME24} & \textbf{AIME25} & \textbf{HMMT25} & \textbf{Avg.} \\
    \midrule
    \multicolumn{9}{l}{\textit{Baselines}} \\
    \quad Base & 80.0 & 73.9 & 49.2 & 67.7 & 29.7 & 23.3 & 12.2 & 21.8 \\
    \quad GRPO$_{\text{think}}$ & 80.8$\pm$1.1 & \textbf{74.0$\pm$0.6} & 50.1$\pm$1.0 & 68.3 & 31.1$\pm$1.1 & 24.6$\pm$0.5 & 12.6$\pm$0.6 & 22.8 \\
    \quad GRPO$_{\text{nothink}}$ & 80.7$\pm$1.7 & 73.5$\pm$0.8 & \textbf{51.2$\pm$1.8} & 68.5 & 30.7$\pm$1.8 & 24.7$\pm$2.4 & 13.6$\pm$0.5 & 23.0 \\
    \quad SFT & 79.7$\pm$1.0 & 71.6$\pm$1.1 & 49.6$\pm$0.9 & 66.9 & 65.7$\pm$1.2 & 54.7$\pm$1.7 & 33.8$\pm$1.4 & 51.4 \\
    \midrule
    \multicolumn{9}{l}{\textit{OPSD (teacher thinking-on)}} \\
    \quad Ref $z$  & 80.8$\pm$0.6 & 72.7$\pm$1.3 & 49.7$\pm$1.5 & 67.7 & 65.3$\pm$3.8 & 51.8$\pm$6.5 & 35.4$\pm$2.4 & 50.8 \\
    \quad No $z$       & \textbf{82.3$\pm$1.1} & 73.1$\pm$1.4 & 51.0$\pm$2.3 & \textbf{68.8} & \textbf{66.3$\pm$3.3} & \textbf{59.9$\pm$1.8} & \textbf{41.0$\pm$2.9} & \textbf{55.7} \\
    \quad Shuffled $z$ & 80.5$\pm$0.7 & 73.9$\pm$1.1 & 49.7$\pm$1.2 & 68.0 & 64.2$\pm$5.8 & 54.7$\pm$5.0 & 37.6$\pm$4.6 & 52.1 \\
    \midrule
    \multicolumn{9}{l}{\textit{OPSD (teacher thinking-off)}} \\
    \quad Ref $z$  & 80.7$\pm$1.0 & 72.6$\pm$1.1 & 50.0$\pm$2.5 & 67.8 & 32.6$\pm$1.4 & 25.6$\pm$0.8 & 13.1$\pm$1.1 & 23.8 \\
    \quad No $z$       & 81.5$\pm$1.1 & 73.9$\pm$0.8 & 50.0$\pm$1.3 & 68.4 & 34.2$\pm$1.9 & 26.8$\pm$0.7 & 16.5$\pm$1.4 & 25.8 \\
    \quad Shuffled $z$ & 80.0$\pm$1.0 & \textbf{74.0$\pm$1.3} & 50.8$\pm$2.0 & 68.3 & 29.2$\pm$2.3 & 24.7$\pm$1.0 & 12.6$\pm$1.7 & 22.2 \\
    \bottomrule
  \end{tabular}
  \end{adjustbox}
\end{table}

\begin{table}[h]
  \caption{Olmo-3-7B-Instruct training methods evaluated under \texttt{avg@12}. OPSD self-distill uses the frozen Olmo-3-7B-Instruct as the teacher, while OPSD cross-model uses Olmo-3-7B-Think as the teacher. Denote $z$ as the reference solution shown to the teacher. We report the checkpoint with the best average performance (averaged over four seeds), reporting mean$\pm$std accuracy.}
  \label{tab:opsdolmo7b}
  \centering
  \setlength{\tabcolsep}{4pt}
  \begin{tabular}{lcccc}
    \toprule
    \textbf{Method} & \textbf{AIME24} & \textbf{AIME25} & \textbf{HMMT25} & \textbf{Avg.} \\
    \midrule
    \multicolumn{5}{l}{\textit{Baselines}} \\
    \quad Base & 50.3 & 37.2 & 23.6 & 37.0 \\
    \quad GRPO & 51.9$\pm$1.3 & 39.2$\pm$1.0 & 25.6$\pm$2.3 & 38.9 \\
    \quad SFT  & 53.1$\pm$0.5 & 40.4$\pm$0.9 & 24.7$\pm$1.2 & 39.4 \\
    \midrule
    \multicolumn{5}{l}{\textit{OPSD (self-distill)}} \\
    \quad Ref $z$  & 55.4$\pm$1.0 & 41.7$\pm$0.4 & 26.4$\pm$0.8 & 41.2 \\
    \quad No $z$       & 54.0$\pm$1.7 & 41.7$\pm$1.3 & 25.3$\pm$0.9 & 40.3 \\
    \quad Shuffled $z$ & 53.3$\pm$1.9 & 40.6$\pm$1.1 & 26.7$\pm$1.2 & 40.2 \\
    \midrule
    \multicolumn{5}{l}{\textit{OPSD (cross-model)}} \\
    \quad Ref $z$  & 64.6$\pm$2.5 & 50.7$\pm$1.1 & 35.9$\pm$2.2 & 50.4 \\
    \quad No $z$       & \textbf{66.1$\pm$1.3} & \textbf{52.3$\pm$2.1} & \textbf{38.1$\pm$1.3} & \textbf{52.2} \\
    \quad Shuffled $z$ & 63.7$\pm$1.1 & 51.0$\pm$1.0 & 34.7$\pm$2.0 & 49.8 \\
    \bottomrule
  \end{tabular}
\end{table}

\input{figures/tables/Qwen3-1.7B_summary}
\input{figures/tables/Qwen3-4B_summary}
\input{figures/tables/Qwen3-8B_summary}
\input{figures/tables/Qwen3-14B_summary}
\input{figures/tables/Olmo-3-7B-Instruct_summary}

\FloatBarrier

\subsection{OPSD example generations}
\begin{figure}[h]
\centering
\begin{minipage}{0.96\linewidth}
\small
\setlength{\parindent}{0pt}
\setlength{\parskip}{0pt}

\vspace{0.5em}

Problem: There exist real numbers $x$ and $y$, both greater than 1, such that $\log_x\left(y^x\right)=\log_y\left(x^{4y}\right)=10$. Find $xy$.

\vspace{0.6em}
Model output (thinking-off):

\vspace{0.3em}
\ttfamily
We are given the following equations involving logarithms: \\
\mbox{} \\
\$\$ \\
\textbackslash{}log\_x(y\^{}x) = 10 \textbackslash{}quad \textbackslash{}text\{and\} \textbackslash{}quad \textbackslash{}log\_y(x\^{}\{4y\}) = 10 \\
\$\$ \\
\mbox{} \\
We are to find the value of \$ xy \$, where \$ x \textgreater{} 1 \$ and \$ y \textgreater{} 1 \$ are real numbers. \\
\mbox{} \\
\texttt{[... solution continues ...]} \\
\mbox{} \\
\#\#\# Step 3: Multiply equations (3) and (4) \\
\mbox{} \\
Let\textquoteright{}s multiply both equations together: \\
\mbox{} \\
\$\$ \\
\textbackslash{}left(x \textbackslash{}cdot \textbackslash{}frac\{\textbackslash{}log y\}\{\textbackslash{}log x\} \textbackslash{}right) \textbackslash{}cdot \textbackslash{}left(4y \textbackslash{}cdot \textbackslash{}frac\{\textbackslash{}log x\}\{\textbackslash{}log y\} \textbackslash{}right) = 10 \textbackslash{}cdot 10 = 100 \\
\$\$ \\
\mbox{} \\
Simplify the left-hand side: \\
\mbox{} \\
\$\$ \\
x \textbackslash{}cdot 4y \textbackslash{}cdot \textbackslash{}frac\{\textbackslash{}log y\}\{\textbackslash{}log x\} \textbackslash{}cdot \textbackslash{}frac\{\textbackslash{}log x\}\{\textbackslash{}log y\} = 4xy \\
\$\$ \\
\mbox{} \\
So: \\
\mbox{} \\
\$\$ \\
4xy = 100 \textbackslash{}Rightarrow xy = \textbackslash{}frac\{100\}\{4\} = 25 \\
\$\$ \\
\mbox{} \\
--- \\
\mbox{} \\
\#\#\# Final Answer: \\
\mbox{} \\
\$\$ \\
\textbackslash{}boxed\{25\} \\
\$\$

\end{minipage}

\caption{Base Qwen3-8B thinking-off rollout on an AIME24 problem generated with temperature $1.0$. The chat template opens a \texttt{<think>} block, the model reasons in the first person, tests several numerical values of $x$ before finding a one-line argument, and closes the block with a single matched \texttt{</think>} before writing a polished solution. Exactly one \texttt{<think>} and one \texttt{</think>} appear in the generation. The full generation is $559$ tokens long.}
\label{fig:base8bthinkoff}
\end{figure}

\begin{figure}[h]
\centering
\begin{minipage}{0.96\linewidth}
\small
\setlength{\parindent}{0pt}
\setlength{\parskip}{0pt}

\vspace{0.5em}

Problem: There exist real numbers $x$ and $y$, both greater than 1, such that $\log_x\left(y^x\right)=\log_y\left(x^{4y}\right)=10$. Find $xy$.

\vspace{0.6em}
Model output (thinking-on):

\vspace{0.3em}
\ttfamily
<think> \\
Okay, so I need to solve this problem where there are real numbers x and y, both greater than 1, such that log base x of y\^{}x equals log base y of x\^{}\{4y\} equals 10. And I need to find the product xy. Alright, let me try to unpack this step by step. \\
\mbox{} \\
First, let me recall some logarithm properties. The logarithm log\_b(a) is equal to ln(a)/ln(b) if I remember correctly. Also, there's a power rule: log\_b(a\^{}c) = c*log\_b(a). \texttt{[... reasoning continues ...]} \\
\mbox{} \\
What if x is smaller? Let's try x = 2. Let me check again. Wait, tried x = 2, got around 7.28. Let me try x = 1.5. Then y = 25 / 1.5 $\approx$ 16.666... \\
\mbox{} \\
Compute log\_\{1.5\}(y\^{}1.5) = 1.5 * log\_\{1.5\}(y). log\_\{1.5\}(y) = log\_\{1.5\}(25 / 1.5) = log\_\{1.5\}(25) - log\_\{1.5\}(1.5) = log\_\{1.5\}(25) - 1. Let me compute log\_\{1.5\}(25). Let me use natural logs: ln(25)/ln(1.5) $\approx$ 3.2189 / 0.4055 $\approx$ 7.936. So log\_\{1.5\}(y) $\approx$ 7.936 - 1 = 6.936. Multiply by 1.5: $\approx$ 10.404. That's close to 10. Maybe with x slightly larger than 1.5? \\
\mbox{} \\
Wait, so with x = 1.5, the value is approximately 10.4, which is slightly above 10. Let me try x = 1.6. \texttt{[... solution continues ...]} \\
\mbox{} \\
Therefore, I think that's solid. I can't find any mistake in this reasoning. So I think the answer is 25. \\
\mbox{} \\
**Final Answer** \\
\textbackslash{}boxed\{25\} \\
</think> \\
\mbox{} \\
To solve the given problem, we are given two equations involving logarithms: \\
\mbox{} \\
\$\$ \\
\textbackslash{}log\_x(y\^{}x) = 10 \textbackslash{}quad \textbackslash{}text\{and\} \textbackslash{}quad \textbackslash{}log\_y(x\^{}\{4y\}) = 10 \\
\$\$ \\
\mbox{} \\
Our goal is to find the value of \$ xy \$, where \$ x \$ and \$ y \$ are both greater than 1. \\
\mbox{} \\
--- \\
\mbox{} \\
\#\#\# Step 1: Use the Power Rule of Logarithms \\
\mbox{} \\
\texttt{[... solution continues ...]} \\
Thus, the result is consistent across all derivations. \\
\mbox{} \\
--- \\
\mbox{} \\
\#\#\# Final Answer \\
\mbox{} \\
\$\$ \\
\textbackslash{}boxed\{25\} \\
\$\$

\end{minipage}

\caption{Base Qwen3-8B thinking-on rollout on an AIME24 problem generated with temperature $1.0$. The chat template opens a \texttt{<think>} block, the model reasons in the first person, tests several numerical values of $x$ before finding a one-line argument, and closes the block with a single matched \texttt{</think>} before writing a polished solution. Exactly one \texttt{<think>} and one \texttt{</think>} appear in the generation. The generation is $5,067$ tokens.}
\label{fig:base8bthinkon}
\end{figure}

\begin{figure}[h]
\centering
\begin{minipage}{0.96\linewidth}
\small
\setlength{\parindent}{0pt}
\setlength{\parskip}{0pt}

\vspace{0.5em}

Problem: There exist real numbers $x$ and $y$, both greater than 1, such that $\log_x\left(y^x\right)=\log_y\left(x^{4y}\right)=10$. Find $xy$.

\vspace{0.6em}
Model output:

\vspace{0.3em}
\ttfamily
Not sure what to do with this one. Maybe I can try to find the value of x and y? \\
\mbox{} \\
Okay, the problem says there exist real numbers \$ x \$ and \$ y \$, both greater than 1, such that: 
\mbox{} \\
\texttt{[... reasoning continues ...]} \\
\mbox{} \\
Wait, that seems straightforward. Did I do everything right? \\
Let me check again. \\
\mbox{} \\
\texttt{[... reasoning continues ...]} \\
\mbox{} \\
This is a transcendental equation and may not have an algebraic solution, but the problem says such real numbers exist, so it's sufficient that we found \$ xy = 25 \$, and the problem doesn't ask us to find specific x and y. \\
\mbox{} \\
Therefore, my conclusion is that \$ xy = 25 \$. \\
\mbox{} \\
So the answer should be \textbackslash{}boxed\{25\} \\
\mbox{} \\
**Final Answer** \\
\textbackslash{}boxed\{25\} \\
</think> \\
\mbox{} \\
To solve the problem, we are given two expressions involving logarithms: \\
\mbox{} \\
\$\$ \textbackslash{}log\_x(y\textasciicircum{}x) = \textbackslash{}log\_y(x\textasciicircum{}\{4y\}) = 10 \$\$ \\
\mbox{} \\
We are to determine the value of the product \$ xy \$, where both \$ x \$ and \$ y \$ are real numbers greater than 1. \\
\mbox{} \\
--- \\
\mbox{} \\
\#\#\# Step 1: Simplify the First Logarithmic Expression \\
\mbox{} \\
We begin by simplifying the logarithmic expression: \\
\mbox{} \\
\$\$ \textbackslash{}log\_x(y\textasciicircum{}x) \$\$ \\
\mbox{} \\
Using the logarithmic identity: \\
\mbox{} \\
\$\$ \textbackslash{}log\_b(a\textasciicircum{}n) = n \textbackslash{}cdot \textbackslash{}log\_b(a) \$\$ \\
\mbox{} \\
\texttt{[... solution continues ...]} \\
\mbox{} \\
\#\#\# Step 5: Final Answer \\
\mbox{} \\
We conclude that the product \$ xy \$ is: \\
\mbox{} \\
\$\$ \textbackslash{}boxed\{25\} \$\$

\end{minipage}

\caption{Qwen3-8B emits an orphan \texttt{</think>} token under thinking-off evaluation. The model was trained with OPSD using a thinking-mode teacher and matched privileged information z and this rollout was generated from checkpoint 160 on an AIME24 problem, taking $2,529$ tokens total. Although thinking is disabled (the chat template already supplies an empty \texttt{<think></think>} block; see \cref{app:qwenthink}), the response opens with an extended first-person reasoning trace, boxes an answer, emits a second, unmatched \texttt{</think>}, and only then writes a polished solution, reproducing the structure of a thinking-on response.}
\label{fig:orphantokenex}
\end{figure}

\FloatBarrier
\subsection{OPSD masking results} \label{app:opsdmasking}
To isolate LLM reasoning and grading improvements from unrelated distillation in, we do ablations for Qwen3-8B confining losses to the first $90\%$ of the student rollout and last $10\%$ of the student rollout.
\input{figures/tables/opsdmasking}

\FloatBarrier

\subsection{Gradient alignment results} \label{app:gradprobecont}
We report gradient cosine similarities for the remaining Qwen3 models in \cref{tab:gradcos_opsd_all} and for Olmo-3-7B-Instruct in \cref{tab:gradcos_opsd_olmo}, following \cref{app:gradprobe}. \Cref{tab:gns_math,tab:gns_lcb} report the gradient noise scale of every gradient.

\begin{table}[h]
  \caption{Cosine similarity between the gradient of each method trained on OpenThoughts and the GRPO$_{\text{nothink}}$ gradient on AIME24, AIME25, and HMMT25, measured at the base model (pooled over four seeds, $\pm$ paired jackknife error). Teacher thinking-on/off denotes whether thinking was enabled for the OPSD teacher, and $z$ denotes the reference solution shown to the teacher.}
  \label{tab:gradcos_opsd_all}
  \centering
  \setlength{\tabcolsep}{4pt}
  \begin{adjustbox}{center,max width=0.9\textwidth}
  \begin{tabular}{lcccc}
    \toprule
    \textbf{Method} & \textbf{Qwen3-1.7B} & \textbf{Qwen3-4B} & \textbf{Qwen3-8B} & \textbf{Qwen3-14B} \\
    \midrule
    \multicolumn{5}{l}{\textit{Baselines}} \\
    \quad GRPO$_{\text{think}}$ & $-0.008\pm0.013$ & $+0.008\pm0.017$ & $-0.014\pm0.019$ & $-0.004\pm0.027$ \\
    \quad GRPO$_{\text{nothink}}$ & $+0.018\pm0.047$ & $+0.060\pm0.050$ & $+0.118\pm0.050$ & $+0.166\pm0.029$ \\
    \quad SFT & $-0.153\pm0.026$ & $-0.056\pm0.021$ & $-0.075\pm0.031$ & $-0.056\pm0.015$ \\
    \midrule
    \multicolumn{5}{l}{\textit{OPSD (teacher thinking-on)}} \\
    \quad Ref $z$ & $-0.205\pm0.058$ & $-0.132\pm0.056$ & $-0.160\pm0.020$ & $-0.075\pm0.017$ \\
    \quad No $z$ & $-0.204\pm0.062$ & $-0.134\pm0.041$ & $-0.157\pm0.019$ & $-0.071\pm0.018$ \\
    \quad Shuffled $z$ & $-0.198\pm0.057$ & $-0.139\pm0.052$ & $-0.157\pm0.019$ & $-0.067\pm0.014$ \\
    \midrule
    \multicolumn{5}{l}{\textit{OPSD (teacher thinking-off)}} \\
    \quad Ref $z$ & $-0.185\pm0.058$ & $-0.086\pm0.056$ & $-0.163\pm0.013$ & $-0.090\pm0.023$ \\
    \quad No $z$ & $-0.087\pm0.068$ & $-0.145\pm0.064$ & $-0.172\pm0.018$ & $-0.087\pm0.025$ \\
    \quad Shuffled $z$ & $-0.157\pm0.060$ & $-0.076\pm0.053$ & $-0.163\pm0.011$ & $-0.093\pm0.021$ \\
    \bottomrule
  \end{tabular}
  \end{adjustbox}
\end{table}

\begin{table}[h]
  \caption{Cosine similarity between the gradient of each method trained on OpenThoughts and the GRPO gradient on AIME24, AIME25, and HMMT25 for Olmo-3-7B-Instruct, measured at the base model (pooled over four seeds, $\pm$ paired jackknife error). OPSD self-distill uses the frozen Olmo-3-7B-Instruct as the teacher, while OPSD cross-model uses Olmo-3-7B-Think as the teacher.}
  \label{tab:gradcos_opsd_olmo}
  \centering
  \setlength{\tabcolsep}{4pt}
  \begin{tabular}{lc}
    \toprule
    \textbf{Method} & \textbf{Cosine} \\
    \midrule
    \multicolumn{2}{l}{\textit{Baselines}} \\
    \quad GRPO & $+0.026\pm0.021$ \\
    \quad SFT & $-0.060\pm0.021$ \\
    \midrule
    \multicolumn{2}{l}{\textit{OPSD (self-distill)}} \\
    \quad Ref $z$ & $+0.057\pm0.052$ \\
    \quad No $z$ & $+0.051\pm0.047$ \\
    \quad Shuffled $z$ & $+0.062\pm0.054$ \\
    \midrule
    \multicolumn{2}{l}{\textit{OPSD (cross-model)}} \\
    \quad Ref $z$ & $+0.052\pm0.031$ \\
    \quad No $z$ & $+0.056\pm0.028$ \\
    \quad Shuffled $z$ & $+0.052\pm0.028$ \\
    \bottomrule
  \end{tabular}
\end{table}

\begin{table}[h]
  \caption{Gradient noise scale across four seeds for the math probe, where each seed's gradient is the mean over 64 iterations; $0$ is a noiseless estimate. The reference is GRPO$_{\text{nothink}}$ on AIME24, AIME25, and HMMT25. For Olmo-3-7B-Instruct, the teacher thinking-on and thinking-off rows are replaced by OPSD cross-model (Olmo-3-7B-Think teacher) and self-distill, respectively.}
  \label{tab:gns_math}
  \centering
  \setlength{\tabcolsep}{4pt}
  \begin{tabular}{lccccc}
    \toprule
    \textbf{Gradient} & \textbf{Qwen3-1.7B} & \textbf{Qwen3-4B} & \textbf{Qwen3-8B} & \textbf{Qwen3-14B} & \textbf{Olmo-3-7B-Instruct} \\
    \midrule
    Reference & $3.583$ & $3.436$ & $3.666$ & $3.736$ & $4.033$ \\
    \midrule
    \multicolumn{6}{l}{\textit{Baselines}} \\
    \quad GRPO$_{\text{think}}$ & $3.979$ & $4.108$ & $3.889$ & $3.786$ & -- \\
    \quad GRPO$_{\text{nothink}}$ & $3.653$ & $3.483$ & $3.322$ & $2.717$ & $3.951$ \\
    \quad SFT & $0.000$ & $0.000$ & $0.001$ & $0.001$ & $0.002$ \\
    \midrule
    \multicolumn{6}{l}{\textit{OPSD (teacher thinking-on / cross-model)}} \\
    \quad Ref $z$ & $0.001$ & $0.001$ & $0.002$ & $0.001$ & $0.008$ \\
    \quad No $z$ & $0.002$ & $0.002$ & $0.003$ & $0.002$ & $0.011$ \\
    \quad Shuffled $z$ & $0.001$ & $0.001$ & $0.001$ & $0.001$ & $0.010$ \\
    \midrule
    \multicolumn{6}{l}{\textit{OPSD (teacher thinking-off / self-distill)}} \\
    \quad Ref $z$ & $0.003$ & $0.004$ & $0.003$ & $0.004$ & $0.010$ \\
    \quad No $z$ & $0.010$ & $0.008$ & $0.010$ & $0.007$ & $0.017$ \\
    \quad Shuffled $z$ & $0.004$ & $0.004$ & $0.002$ & $0.006$ & $0.011$ \\
    \bottomrule
  \end{tabular}
\end{table}

\begin{table}[h]
  \caption{Gradient noise scale across four seeds for the LiveCodeBench v6 probe, where each seed's gradient is the mean over 8 iterations; $0$ is a noiseless estimate. The reference is GRPO on \textsc{test\_unseen}; GRPO and SDPO are trained on \textsc{train\_seen}.}
  \label{tab:gns_lcb}
  \centering
  \setlength{\tabcolsep}{4pt}
  \begin{tabular}{lccc}
    \toprule
    \textbf{Gradient} & \textbf{Qwen3-8B} & \textbf{Qwen3-14B} & \textbf{Olmo-3-7B-Instruct} \\
    \midrule
    Reference & $1.968$ & $3.152$ & $1.999$ \\
    GRPO & $2.855$ & $3.125$ & $2.076$ \\
    SDPO & $0.043$ & $0.018$ & $0.020$ \\
    \bottomrule
  \end{tabular}
\end{table}

\FloatBarrier

\section{Implementation details}\label{sec:implementation}
All our experiments are run were run with \texttt{8xH100} nodes (with no multi-node parallelism), primarily leveraging the vLLM inference engine \citep{vllm}. We run our experiments on top of the OPSD codebase written in TRL \citep{trl} and the SDPO codebase which uses veRL \citep{verl}. Our SDPO LCBv6 runs train for $\sim$40 hours, while the SDPO science and tool use jobs took $\sim$3 hours. OPSD runs took up to $\sim$2 hours and GRPO runs on the OpenThoughts dataset took up to $\sim$6 hours. OPSD thinking-on evaluations took up to $\sim$2.5 hours per dataset and OPSD thinking-off evaluations took up to $\sim$1 hour per dataset. Our gradient similarity analyses took up to $\sim$2 hours per training setting.

\subsection{SDPO hyperparameters}\label{subsec:sdpo_implementation}

\paragraph{LiveCodeBench v6.}
We summarize hyperparameters used for in the LCBv6 experiments in \cref{tab:sdpolcbhyperparams}, mirroring the hyperparameters taken from the SDPO setup~\citep{sdpo}. We use the exact same rich feedback settings and prompts as \citet{sdpo}, changing only the dataset splits.

\begin{table}[h]
  \centering
  \caption{SDPO and GRPO training hyperparameters on LiveCodeBench v6.}
  \setlength{\tabcolsep}{4pt}
    \begin{adjustbox}{center,max width=0.85\textwidth}
  \begin{tabular}{lcc}
    \toprule
    \textbf{Parameter} & \textbf{GRPO} & \textbf{SDPO} \\
    \midrule
    \multicolumn{3}{l}{\textit{General}} \\
    \quad Model & \{Qwen3-8/14B, Olmo-3-7B-Instruct\} & \{Qwen3-8/14B, Olmo-3-7B-Instruct\} \\
    \quad Thinking & False & False \\
    \quad Total steps & 200 & 200 \\

    \multicolumn{3}{l}{\textit{Data}} \\
    \quad Max. prompt length & 2048 & 2048 \\

    \multicolumn{3}{l}{\textit{Batching}} \\
    \quad Question batch size & 32 & 32 \\
    \quad Mini batch size & 8 & 1 \\
    \quad Number of rollouts & 8 & 8 \\

    \multicolumn{3}{l}{\textit{Rollout}} \\
    \quad Inference engine & vLLM & vLLM \\
    \quad Temperature & 1.0 & 1.0 \\
    \quad Max. response length & 8192 & 8192 \\

    \multicolumn{3}{l}{\textit{Validation}} \\
    \quad Number of rollouts & 4 & 4 \\
    \quad Temperature & 0.6 & 0.6 \\
    \quad Top-$p$ & 0.95 & 0.95 \\
    \quad Max. response length & 8192 & 8192 \\

    \multicolumn{3}{l}{\textit{SDPO loss}} \\
    \quad Top-$K$ distillation & - & 20 \\
    \quad Distillation divergence & - & Reverse-KL \\
    \quad Clip advantages & - & - \\
    \quad Teacher-EMA update rate & - & 0.01 \\
    \quad Rollout importance sampling clip & - & 2.0 \\

    \multicolumn{3}{l}{\textit{GRPO loss}} \\
    \quad $\epsilon$-high & 0.28 & - \\
    \quad Rollout importance sampling clip & 2.0 & - \\
    \quad KL coefficient ($\lambda$) & 0.0 & - \\

    \multicolumn{3}{l}{\textit{Training}} \\
    \quad Optimizer & AdamW & AdamW \\
    \quad Learning rate & $1\times 10^{-6}$ & $1\times 10^{-6}$ \\
    \quad Warmup steps & 0 & 0 \\
    \quad Weight decay & 0.01 & 0.01 \\
    \quad Gradient clip norm & 1.0 & 1.0 \\

    \bottomrule
  \end{tabular}\label{tab:sdpolcbhyperparams}
  \end{adjustbox}
\end{table}

\paragraph{SciKnowEval and ToolAlpaca.}
We also follow the SDPO and GRPO hyperparameters used in \citet{sdpo} for SciKnowEval and ToolAlpaca, except that we find that GRPO is unstable with learning rate $1\times 10^{-5}$, we report results with $1\times 10^{-6}$, running SDPO and GRPO in an on-policy manner. We summarize these hyperparameters in \cref{tab:sdposciknowevaltoolalpacahyperparams}.

\begin{table}[h]
  \centering
  \caption{GRPO and SDPO training hyperparameters on SciKnowEval and ToolAlpaca}
  \setlength{\tabcolsep}{4pt}
  \begin{tabular}{lcc}
    \toprule
    \textbf{Parameter} & \textbf{GRPO} & \textbf{SDPO} \\
    \midrule
    \multicolumn{3}{l}{\textit{General}} \\
    \quad Model & Qwen3-8B & Qwen3-8B \\
    \quad Thinking & False & False \\
    \quad Total steps & 200 & 200 \\

    \multicolumn{3}{l}{\textit{Data}} \\
    \quad Max. prompt length & 2048 & 2048 \\

    \multicolumn{3}{l}{\textit{Batching}} \\
    \quad Question batch size & 32 & 32 \\
    \quad Mini batch size &  32 & 32 \\
    \quad Number of rollouts & 8 & 8 \\

    \multicolumn{3}{l}{\textit{Rollout}} \\
    \quad Inference engine & vLLM & vLLM \\
    \quad Temperature & 1.0 & 1.0 \\
    \quad Max. response length & 8192 & 8192 \\

    \multicolumn{3}{l}{\textit{Validation}} \\
    \quad Number of rollouts & 16 & 16 \\
    \quad Temperature & 0.6 & 0.6 \\
    \quad Top-$p$ & 0.95 & 0.95 \\
    \quad Max. response length & 8192 & 8192 \\

    \multicolumn{3}{l}{\textit{SDPO loss}} \\
    \quad Top-$K$ distillation & - & 100 \\
    \quad Distillation divergence & - & Jensen-Shannon \\
    \quad Clip advantages & - & - \\
    \quad Teacher-EMA update rate & - & 0.05 \\
    \quad Rollout importance sampling clip & - & 2.0 \\

    \multicolumn{3}{l}{\textit{Training}} \\
    \quad Optimizer & AdamW & AdamW \\
    \quad Learning rate & $1\times 10^{-6}$ & $1\times 10^{-5}$ \\
    \quad Warmup steps & 0 & 0 \\
    \quad Weight decay & 0.01 & 0.01 \\
    \quad Gradient clip norm & 1.0 & 1.0 \\

    \bottomrule
  \end{tabular}\label{tab:sdposciknowevaltoolalpacahyperparams}
\end{table}

\subsection{OPSD hyperparameters}\label{app:opsdhyperparams}
We report our hyperparameters in \cref{tab:opsdhyperparams}; we follow the training configurations provided by \citet{opsd}, training on OpenThoughts \citep{openthoughts}, changing only the content of $z$. The one exception is training length: we run OPSD and SFT for 100 steps longer than \citet{opsd} to evaluate fully converged performance.

\begin{table}[h]
  \centering
  \caption{GRPO, OPSD, and SFT training hyperparameters on OpenThoughts. $^\dagger$Olmo-3-7B-Instruct has no thinking toggle, so thinking settings do not apply: it is trained and evaluated in its native mode, and its OPSD teacher is either the frozen student (self-distillation) or Olmo-3-7B-Think (cross-model).}
  \setlength{\tabcolsep}{4pt}
\begin{adjustbox}{center,max width=0.85\textwidth}
  \begin{tabular}{lccc}
    \toprule
    \textbf{Parameter} & \textbf{GRPO} & \textbf{OPSD} & \textbf{SFT} \\
    \midrule
    \multicolumn{4}{l}{\textit{General}} \\
    \quad Model & \multicolumn{3}{c}{Qwen3-\{1.7, 4, 8, 14\}B, Olmo-3-7B-Instruct} \\
    \quad Student-thinking & \{True, False\}$^\dagger$ & False$^\dagger$ & False$^\dagger$ \\
    \quad Teacher-thinking & - & \{True, False\}$^\dagger$ & - \\
    \quad Total steps & 500 & 200 & 200 \\

    \multicolumn{4}{l}{\textit{Data}} \\
    \quad Max. prompt length & 2048 & 2048 & 2048 \\

    \multicolumn{4}{l}{\textit{Batching}} \\
    \quad Question batch size & 4 & 32 & 32 \\
    \quad Number of rollouts & 8 & 1 & - \\
    \quad Optimizer steps per generation &  2 & 1  & - \\

    \multicolumn{4}{l}{\textit{Rollout}} \\
    \quad Inference engine & vLLM & vLLM & vLLM \\
    \quad Temperature & 1.2 & 1.1 & - \\
    \quad Top-$p$ & 1.0 & 0.95 & - \\
    \quad Top-$k$ & - & 20 & - \\
    \quad Max. response length & 16000 & 1024 & - \\

    \multicolumn{4}{l}{\textit{Validation}} \\
    \quad Student-thinking & \{True, False\}$^\dagger$ & \{True, False\}$^\dagger$ & \{True, False\}$^\dagger$\\
    \quad Number of rollouts & 12 & 12 & 12 \\
    \quad Temperature & 1.0 & 1.0 & 1.0 \\
    \quad Top-$p$ & \multicolumn{3}{c}{0.95 (thinking-on), 0.8 (thinking-off), 0.8 (Olmo-3-7B-Instruct)} \\
    \quad Top-$k$ & - & - & - \\
    \quad Max. response length & 38912 & 38912 & 38912 \\

    \multicolumn{4}{l}{\textit{OPSD loss}} \\
    \quad Teacher policy & - & Frozen student or Olmo-3-7B-Think$^\dagger$ & - \\
    \quad Teacher $z$ & - & \{ref $z$, no $z$, shuffled $z$\} & - \\
    \quad Distillation divergence & - & Forward KL & - \\
    \quad Token clip (Qwen3-1.7B, 4B) & - & 0.05 & - \\
    \quad Token clip (Qwen3-8B, 14B) & - & 0.06 & - \\
    \quad Token clip (Olmo-3-7B-Instruct) & - & 0.06 & - \\

    \multicolumn{4}{l}{\textit{Training}} \\
    \quad Optimizer & AdamW & AdamW & AdamW \\
    \quad Learning rate & $5\times 10^{-6}$ & $5\times 10^{-6}$ &  $5\times 10^{-6}$\\
    \quad Warmup steps & 0 & 0 & 0 \\
    \quad Weight decay & 0 & 0 & 0 \\
    \quad Gradient clip norm & 1.0 & 0.1 & 1.0\\
    \quad LoRA rank & 64 & 64 & 64 \\
    \quad LoRA alpha & 128 & 128 & 128 \\
    \bottomrule
  \end{tabular}\label{tab:opsdhyperparams}
\end{adjustbox}
\end{table}

\subsection{Gradient similarity analysis}\label{app:gradprobe}
We record each method's gradients by running its training loop from the base model with learning rate zero, so the weights never change, and averaging the gradients over all iterations of each run for downstream computation. We use the hyperparameters of \cref{tab:sdpolcbhyperparams,tab:opsdhyperparams}, with two exceptions: GRPO uses \texttt{num\_iterations}$=1$, since at learning rate zero a second pass over the same rollouts gives the same gradient, and we record gradients with respect to all parameters rather than the LoRA adapters (for the math results), since LoRA's zero-initialized matrices may lead to poorly defined gradient behavior. We run 64 iterations per seed for OPSD on OpenThoughts and 8 for SDPO on LiveCodeBench v6, with four seeds per method.

We define the following quantities for use in our analysis:

\textbf{Pooled cosine similarity.} Given $S=4$ seeds, let $\bar{g}_s$ denote the gradient of seed $s$ averaged over its iterations, and $\bar{g} = \frac{1}{S}\sum_{s=1}^{S} \bar{g}_s$ the pooled gradient. We report the cosine similarity between the pooled gradients of the reference and of each method:
\begin{equation}
    c = \frac{\bar{g}^{\mathrm{ref}}\cdot \bar{g}^{\mathrm{method}} }{\|\bar{g}^{\mathrm{ref}}\|_2 \, \|\bar{g}^{\mathrm{method}}\|_2}.
\end{equation}

\textbf{Jackknife error.} Let $c_{-s}$ denote the pooled cosine similarity computed with seed $s$ removed from both the reference and the method, and $\bar{c}$ the mean of the $c_{-s}$. The jackknife error is then
\begin{equation}
    \mathrm{SE} = \sqrt{\frac{S-1}{S}\sum_{s=1}^{S} \left(c_{-s}-\bar{c}\right)^2}.
\end{equation}

\textbf{Gradient noise scale.} Following \citet{gradientnoisescale}, we measure the noise of each pooled gradient as
\begin{equation}
    \frac{\operatorname{tr}(\Sigma)}{\|\bar{g}\|_2^2},
\end{equation}
where $\Sigma$ is the covariance of $\bar{g}_1,\dots,\bar{g}_S$. It is $0$ for a noiseless estimate.

\subsection{OPSD teacher and student prompts} \label{app:opsdprompts}
We use the standard prompts used by \citet{opsd}.

\begin{figure}[h]
\centering
\begin{minipage}{0.96\linewidth}
\small
\setlength{\parindent}{0pt}
\setlength{\parskip}{0pt}

\textbf{Student Prompt}

\vspace{0.25em}
\hrule
\vspace{0.5em}

\ttfamily
Problem: \{problem\}
\vspace{0.5em}

Please reason step by step, and put your final answer within
\textbackslash boxed\{\}.
\end{minipage}

\caption{Student prompt for OPSD.}
\label{fig:student-prompt}
\end{figure}

\begin{figure}[h]
\centering
\begin{minipage}{0.96\linewidth}
\small
\setlength{\parindent}{0pt}
\setlength{\parskip}{0pt}

\textbf{Privileged Teacher Prompt}

\vspace{0.25em}
\hrule
\vspace{0.5em}

\ttfamily
Problem: \{problem\}

\vspace{0.5em}
Here is a reference solution to this problem:\\
=== Reference Solution Begin ===\\
\{z\}\\
=== Reference Solution End ===

\vspace{1em}
After reading the reference solution above, make sure you truly
understand the reasoning behind each step---do not copy or paraphrase it. Now, using your own words and independent reasoning, derive the same final answer to the problem above. Think step by step, explore different approaches, and don't be afraid to backtrack or reconsider if something doesn't work out:

\vspace{0.5em}
Please reason step by step, and put your final answer within
\textbackslash boxed\{\}.
\end{minipage}

\caption{Teacher prompt template for OPSD. The privileged information ablations ``ref $z$", ``no $z$", and ``shuffled $z$" differ only by the content of what ``\{$z$\}" is replaced by in the template.}
\label{fig:teacher-prompt-ref}
\end{figure}

\end{document}

%% file: math_commands.tex
\usepackage{amsmath,amsfonts,bm}

\def\eqref#1{equation~\ref{#1}}

\def\1{\bm{1}}

\DeclareMathAlphabet{\mathsfit}{\encodingdefault}{\sfdefault}{m}{sl}
\SetMathAlphabet{\mathsfit}{bold}{\encodingdefault}{\sfdefault}{bx}{n}

\newcommand{\KL}{D_{\mathrm{KL}}}

%% file: figures/tables/Qwen3-1.7B_summary.tex
\begin{table}[h]
  \caption{Qwen3-1.7B average response length and percentage of rollouts with an orphan \texttt{</think>} tokens generated per method. Statistics are pooled over the rollouts generated from the existing AIME 24, AIME 25, and HMMT 25 evaluations at the best checkpoint when performance was averaged over four seeds. We report mean$\pm$std over seeds. Tok@best refers to the average response length in tokens of the best checkpoint, Tok@200 is the same quantity at checkpoint $200$, and Orph.\% refers to the fraction of thinking-off rollouts generated by the best checkpoint that emit a \texttt{</think>} with no opener of their own (we do not compute this for thinking-on as the model emits its own \texttt{<think>} and the closer is always matched).}
  \label{tab:qwen317b_summary}
  \centering
  \setlength{\tabcolsep}{4pt}
  \begin{adjustbox}{center,max width=\textwidth}
  \begin{tabular}{lccccc}
    \toprule
     & \multicolumn{2}{c}{\textbf{Thinking-on eval}} & \multicolumn{3}{c}{\textbf{Thinking-off eval}} \\
    \cmidrule(lr){2-3} \cmidrule(lr){4-6}
    \textbf{Method} & \textbf{Tok@200} & \textbf{Tok@best} & \textbf{Orph.\,\%} & \textbf{Tok@200} & \textbf{Tok@best} \\
    \midrule
    \multicolumn{6}{l}{\textit{Baselines}} \\
    \quad Base & 18413 & 18413 & 0.0 & 2889 & 2889 \\
    \quad GRPO$_{\text{think}}$ & 18220$\pm$1459 & 18356$\pm$743 & 0.0$\pm$0.0 & 2985$\pm$374 & 2827$\pm$145 \\
    \quad GRPO$_{\text{nothink}}$ & 19089$\pm$1475 & 19517$\pm$1849 & 0.0$\pm$0.0 & 3161$\pm$874 & 3187$\pm$580 \\
    \quad SFT & 14646$\pm$219 & 20023$\pm$160 & 48.4$\pm$1.5 & 15438$\pm$118 & 15438$\pm$118 \\
    \midrule
    \multicolumn{6}{l}{\textit{OPSD (teacher thinking-on)}} \\
    \quad Ref $z$ & 20526$\pm$224 & 21301$\pm$271 & 31.7$\pm$1.0 & 10602$\pm$157 & 10602$\pm$157 \\
    \quad No $z$ & 20991$\pm$356 & 20967$\pm$111 & 3.2$\pm$0.8 & 9887$\pm$978 & 9534$\pm$773 \\
    \quad Shuffled $z$ & 19821$\pm$288 & 19933$\pm$102 & 42.2$\pm$1.0 & 11397$\pm$422 & 11456$\pm$243 \\
    \midrule
    \multicolumn{6}{l}{\textit{OPSD (teacher thinking-off)}} \\
    \quad Ref $z$ & 18869$\pm$217 & 20026$\pm$166 & 0.3$\pm$0.2 & 3442$\pm$311 & 4008$\pm$104 \\
    \quad No $z$ & 18553$\pm$295 & 17374$\pm$165 & 0.0$\pm$0.0 & 4816$\pm$3008 & 2381$\pm$197 \\
    \quad Shuffled $z$ & 17619$\pm$45 & 21258$\pm$66 & 0.7$\pm$0.2 & 3121$\pm$208 & 5470$\pm$231 \\
    \bottomrule
  \end{tabular}
  \end{adjustbox}
\end{table}

%% file: figures/tables/Qwen3-4B_summary.tex
\begin{table}[h]
  \caption{Qwen3-4B average response length and orphan \texttt{</think>} tokens generated per method. Statistics are pooled over the rollouts generated from the existing AIME 24, AIME 25, and HMMT 25 evaluations at the best checkpoint when performance was averaged over four seeds. We report mean$\pm$std over seeds. Tok@best refers to the average response length in tokens of the best checkpoint and Tok@200 is the same quantity at checkpoint 200. Orph.\% refers to the fraction of rollouts generated by the best checkpoint with thinking disabled that emit a \texttt{</think>} with no opener of their own (we do not compute this for thinking-on as the model emits its own \texttt{<think>} and the closer is always matched).}
  \label{tab:qwen34b_summary}
  \centering
  \setlength{\tabcolsep}{4pt}
  \begin{adjustbox}{center,max width=\textwidth}
  \begin{tabular}{lccccc}
    \toprule
     & \multicolumn{2}{c}{\textbf{Thinking-on eval}} & \multicolumn{3}{c}{\textbf{Thinking-off eval}} \\
    \cmidrule(lr){2-3} \cmidrule(lr){4-6}
    \textbf{Method} & \textbf{Tok@200} & \textbf{Tok@best} & \textbf{Orph.\,\%} & \textbf{Tok@200} & \textbf{Tok@best} \\
    \midrule
    \multicolumn{6}{l}{\textit{Baselines}} \\
    \quad Base & 17556 & 17556 & 0.0 & 3719 & 3719 \\
    \quad GRPO$_{\text{think}}$ & 17742$\pm$197 & 17456$\pm$106 & 0.0$\pm$0.0 & 3604$\pm$170 & 3746$\pm$277 \\
    \quad GRPO$_{\text{nothink}}$ & 17855$\pm$350 & 18078$\pm$585 & 0.0$\pm$0.0 & 4348$\pm$652 & 4693$\pm$537 \\
    \quad SFT & 13871$\pm$153 & 18273$\pm$220 & 53.0$\pm$1.7 & 14028$\pm$70 & 14028$\pm$70 \\
    \midrule
    \multicolumn{6}{l}{\textit{OPSD (teacher thinking-on)}} \\
    \quad Ref $z$ & 19814$\pm$161 & 19832$\pm$129 & 64.1$\pm$1.6 & 16518$\pm$339 & 16445$\pm$308 \\
    \quad No $z$ & 20711$\pm$205 & 20310$\pm$193 & 20.3$\pm$3.2 & 18560$\pm$759 & 15733$\pm$502 \\
    \quad Shuffled $z$ & 19542$\pm$302 & 19789$\pm$51 & 69.9$\pm$2.6 & 16758$\pm$322 & 16732$\pm$126 \\
    \midrule
    \multicolumn{6}{l}{\textit{OPSD (teacher thinking-off)}} \\
    \quad Ref $z$ & 19666$\pm$161 & 19666$\pm$161 & 1.9$\pm$0.3 & 7138$\pm$209 & 7127$\pm$372 \\
    \quad No $z$ & 18161$\pm$207 & 17082$\pm$114 & 0.0$\pm$0.0 & 4621$\pm$177 & 3395$\pm$126 \\
    \quad Shuffled $z$ & 19368$\pm$145 & 19879$\pm$118 & 2.0$\pm$0.2 & 6670$\pm$135 & 6670$\pm$135 \\
    \bottomrule
  \end{tabular}
  \end{adjustbox}
\end{table}

%% file: figures/tables/Qwen3-8B_summary.tex
\begin{table}[h]
  \caption{Qwen3-8B average response length and orphan \texttt{</think>} tokens generated per method. Statistics are pooled over the rollouts generated from the existing AIME 24, AIME 25, and HMMT 25 evaluations at the best checkpoint when performance was averaged over four seeds. We report mean$\pm$std over seeds. Tok@best refers to the average response length in tokens of the best checkpoint and Tok@200 is the same quantity at checkpoint 200. Orph.\% refers to the fraction of rollouts generated by the best checkpoint with thinking disabled that emit a \texttt{</think>} with no opener of their own (we do not compute this for thinking-on as the model emits its own \texttt{<think>} and the closer is always matched).}
  \label{tab:qwen38b_summary}
  \centering
  \setlength{\tabcolsep}{4pt}
  \begin{adjustbox}{center,max width=\textwidth}
  \begin{tabular}{lccccc}
    \toprule
     & \multicolumn{2}{c}{\textbf{Thinking-on eval}} & \multicolumn{3}{c}{\textbf{Thinking-off eval}} \\
    \cmidrule(lr){2-3} \cmidrule(lr){4-6}
    \textbf{Method} & \textbf{Tok@200} & \textbf{Tok@best} & \textbf{Orph.\,\%} & \textbf{Tok@200} & \textbf{Tok@best} \\
    \midrule
    \multicolumn{6}{l}{\textit{Baselines}} \\
    \quad Base & 18264 & 18264 & 0.0 & 3580 & 3580 \\
    \quad GRPO$_{\text{think}}$ & 18209$\pm$484 & 18160$\pm$417 & 0.0$\pm$0.0 & 3762$\pm$123 & 3762$\pm$123 \\
    \quad GRPO$_{\text{nothink}}$ & 18453$\pm$203 & 18479$\pm$243 & 0.0$\pm$0.0 & 4116$\pm$367 & 4355$\pm$313 \\
    \quad SFT & 17807$\pm$86 & 18511$\pm$254 & 75.0$\pm$1.0 & 14488$\pm$143 & 14456$\pm$206 \\
    \midrule
    \multicolumn{6}{l}{\textit{OPSD (teacher thinking-on)}} \\
    \quad Ref $z$ & 20255$\pm$167 & 20188$\pm$131 & 30.6$\pm$2.5 & 12025$\pm$223 & 9750$\pm$384 \\
    \quad No $z$ & 20174$\pm$123 & 20284$\pm$189 & 40.1$\pm$3.5 & 15493$\pm$447 & 15493$\pm$447 \\
    \quad Shuffled $z$ & 20008$\pm$131 & 19669$\pm$524 & 25.2$\pm$6.3 & 13864$\pm$636 & 13864$\pm$636 \\
    \midrule
    \multicolumn{6}{l}{\textit{OPSD (teacher thinking-off)}} \\
    \quad Ref $z$ & 19460$\pm$151 & 19403$\pm$45 & 4.7$\pm$1.3 & 5870$\pm$110 & 5384$\pm$241 \\
    \quad No $z$ & 19003$\pm$146 & 19003$\pm$146 & 0.0$\pm$0.0 & 6736$\pm$544 & 5944$\pm$230 \\
    \quad Shuffled $z$ & 19385$\pm$157 & 19356$\pm$156 & 6.7$\pm$0.3 & 5710$\pm$252 & 9027$\pm$160 \\
    \bottomrule
  \end{tabular}
  \end{adjustbox}
\end{table}

%% file: figures/tables/Qwen3-14B_summary.tex
\begin{table}[h]
  \caption{Qwen3-14B average response length and orphan \texttt{</think>} tokens generated per method. Statistics are pooled over the rollouts generated from the existing AIME 24, AIME 25, and HMMT 25 evaluations at the best checkpoint when performance was averaged over four seeds. We report mean$\pm$std over seeds. Tok@best refers to the average response length in tokens of the best checkpoint and Tok@200 is the same quantity at checkpoint 200. Orph.\% refers to the fraction of rollouts generated by the best checkpoint with thinking disabled that emit a \texttt{</think>} with no opener of their own (we do not compute this for thinking-on as the model emits its own \texttt{<think>} and the closer is always matched).}
  \label{tab:qwen314b_summary}
  \centering
  \setlength{\tabcolsep}{4pt}
  \begin{adjustbox}{center,max width=\textwidth}
  \begin{tabular}{lccccc}
    \toprule
     & \multicolumn{2}{c}{\textbf{Thinking-on eval}} & \multicolumn{3}{c}{\textbf{Thinking-off eval}} \\
    \cmidrule(lr){2-3} \cmidrule(lr){4-6}
    \textbf{Method} & \textbf{Tok@200} & \textbf{Tok@best} & \textbf{Orph.\,\%} & \textbf{Tok@200} & \textbf{Tok@best} \\
    \midrule
    \multicolumn{6}{l}{\textit{Baselines}} \\
    \quad Base & 16757 & 16757 & 0.0 & 3341 & 3341 \\
    \quad GRPO$_{\text{think}}$ & 16435$\pm$223 & 16447$\pm$144 & 0.0$\pm$0.0 & 3389$\pm$142 & 3438$\pm$79 \\
    \quad GRPO$_{\text{nothink}}$ & 16638$\pm$200 & 16646$\pm$270 & 0.0$\pm$0.0 & 3500$\pm$357 & 3569$\pm$89 \\
    \quad SFT & 17061$\pm$138 & 17597$\pm$160 & 48.4$\pm$1.6 & 13145$\pm$117 & 13119$\pm$74 \\
    \midrule
    \multicolumn{6}{l}{\textit{OPSD (teacher thinking-on)}} \\
    \quad Ref $z$ & 17484$\pm$109 & 17525$\pm$128 & 55.7$\pm$10.4 & 13609$\pm$1780 & 13609$\pm$1780 \\
    \quad No $z$ & 18370$\pm$23 & 18179$\pm$202 & 67.7$\pm$6.2 & 20468$\pm$532 & 19456$\pm$631 \\
    \quad Shuffled $z$ & 17653$\pm$200 & 17794$\pm$142 & 64.6$\pm$17.2 & 16657$\pm$403 & 14496$\pm$832 \\
    \midrule
    \multicolumn{6}{l}{\textit{OPSD (teacher thinking-off)}} \\
    \quad Ref $z$ & 17449$\pm$125 & 17576$\pm$194 & 0.0$\pm$0.0 & 8608$\pm$411 & 5998$\pm$273 \\
    \quad No $z$ & 17460$\pm$139 & 17636$\pm$69 & 0.0$\pm$0.0 & 8522$\pm$1134 & 6526$\pm$791 \\
    \quad Shuffled $z$ & 17327$\pm$239 & 17292$\pm$282 & 0.0$\pm$0.0 & 7903$\pm$459 & 5653$\pm$407 \\
    \bottomrule
  \end{tabular}
  \end{adjustbox}
\end{table}

%% file: figures/tables/Olmo-3-7B-Instruct_summary.tex
\begin{table}[h]
  \caption{Olmo-3-7B-Instruct average response length and orphan \texttt{</think>} tokens generated per method. Statistics are pooled over the rollouts generated from the existing AIME 24, AIME 25, and HMMT 25 evaluations at the best checkpoint when performance was averaged over four seeds. We report mean$\pm$std over seeds. Tok@best refers to the average response length in tokens of the best checkpoint and Tok@200 is the same quantity at checkpoint 200. Orph.\% refers to the fraction of rollouts generated by the best checkpoint with thinking disabled that emit a \texttt{</think>} with no opener of their own. Olmo-3-7B-Instruct's chat template already ignores \texttt{enable\_thinking} (every variant renders an identical prompt with no tag prefilled, hence no thinking-on column), and its native format is to reason with no opener and then emit \texttt{</think>} before a clean final answer, so its base model already scores 63.0\% and emits \texttt{<think>} in 0.0\% of rollouts.}
  \label{tab:olmo37binstruct_summary}
  \centering
  \setlength{\tabcolsep}{4pt}
  \begin{adjustbox}{center,max width=\textwidth}
  \begin{tabular}{lccc}
    \toprule
     & \multicolumn{3}{c}{\textbf{Thinking-off eval}} \\
    \cmidrule(lr){2-4}
    \textbf{Method} & \textbf{Orph.\,\%} & \textbf{Tok@200} & \textbf{Tok@best} \\
    \midrule
    \multicolumn{4}{l}{\textit{Baselines}} \\
    \quad Base &  63.0 & 7732 & 7732 \\
    \quad GRPO &  63.7$\pm$1.7 & 7888$\pm$134 & 7915$\pm$214 \\
    \quad SFT &  69.3$\pm$1.3 & 8794$\pm$175 & 8646$\pm$71 \\
    \midrule
    \multicolumn{4}{l}{\textit{OPSD (self teacher)}} \\
    \quad Ref $z$ &  56.4$\pm$1.3 & 7501$\pm$131 & 8262$\pm$101 \\
    \quad No $z$ &  56.2$\pm$1.0 & 7266$\pm$110 & 8427$\pm$78 \\
    \quad Shuffled $z$ &  55.3$\pm$1.7 & 7897$\pm$121 & 8347$\pm$183 \\
    \midrule
    \multicolumn{4}{l}{\textit{OPSD (teacher Olmo-3-7B-Think)}} \\
    \quad Ref $z$ &  99.1$\pm$0.3 & 14239$\pm$84 & 14239$\pm$84 \\
    \quad No $z$ &  98.4$\pm$0.4 & 14745$\pm$24 & 14745$\pm$24 \\
    \quad Shuffled $z$ &  99.2$\pm$0.4 & 14019$\pm$178 & 14019$\pm$178 \\
    \bottomrule
  \end{tabular}
  \end{adjustbox}
\end{table}

%% file: figures/tables/opsdmasking.tex
\begin{table}[h]
  \label{tab:opsd8b_mask_sharedstep}
  \centering
  \setlength{\tabcolsep}{4pt}
  \begin{adjustbox}{center,max width=\textwidth}
  \begin{tabular}{lcccccccc}
    \toprule
     & \multicolumn{4}{c}{\textbf{Thinking-on eval}} & \multicolumn{4}{c}{\textbf{Thinking-off eval}} \\
    \cmidrule(lr){2-5} \cmidrule(lr){6-9}
    \textbf{Priv.\ info} & \textbf{AIME24} & \textbf{AIME25} & \textbf{HMMT25} & \textbf{Avg.} & \textbf{AIME24} & \textbf{AIME25} & \textbf{HMMT25} & \textbf{Avg.} \\
    \midrule
    \multicolumn{9}{l}{\textit{Last 10\% of tokens}} \\
    \quad Ref.\ $z$ & 77.2$\pm$0.3 & 69.9$\pm$1.5 & 46.9$\pm$1.3 & 64.7 & 33.6$\pm$1.7 & 25.1$\pm$1.6 & 15.1$\pm$0.8 & 24.6 \\
    \quad No $z$ & \textbf{77.8$\pm$1.4} & 69.3$\pm$1.4 & \textbf{48.1$\pm$1.2} & 65.1 & 34.8$\pm$0.9 & 28.4$\pm$1.6 & 16.9$\pm$1.0 & 26.7 \\
    \quad Shuffled $z$ & 76.6$\pm$1.0 & 71.9$\pm$0.9 & 47.6$\pm$1.0 & 65.4 & 31.9$\pm$0.9 & 23.8$\pm$1.6 & 16.2$\pm$2.4 & 24.0 \\
    \midrule
    \multicolumn{9}{l}{\textit{All but last 10\%}} \\
    \quad Ref.\ $z$ & 77.1$\pm$0.5 & 70.8$\pm$2.1 & 47.6$\pm$1.8 & 65.2 & 35.1$\pm$1.2 & 27.6$\pm$1.0 & 17.3$\pm$1.7 & 26.7 \\
    \quad No $z$ & 77.7$\pm$0.6 & \textbf{72.6$\pm$0.9} & 47.5$\pm$0.7 & \textbf{65.9} & \textbf{47.4$\pm$1.8} & \textbf{37.3$\pm$1.5} & \textbf{26.2$\pm$1.2} & \textbf{36.9} \\
    \quad Shuffled $z$ & 77.4$\pm$1.4 & 71.6$\pm$1.1 & 46.2$\pm$1.1 & 65.1 & 46.0$\pm$2.7 & 35.4$\pm$1.6 & 22.6$\pm$2.5 & 34.7 \\
    \bottomrule
  \end{tabular}
  \end{adjustbox}
  \caption{Masking ablation for Qwen3-8B OPSD training runs evaluated with \texttt{avg@12} on AIME24, AIME25, and HMMT25. Blocks refer to which parts of student tokens the distillation loss is applied to (the rest is masked out). Rows within a block are the privileged information $z$ shown to the teacher (the reference solution, one shuffled from another problem, or absent). Teacher thinking was on during OPSD training throughout. Each row contains the numbers aggregated over 4 seeds of a single checkpoint trained using the given privileged information. Accuracy is mean$\pm$std across those seeds. All settings and hyperparameters except masking follow \citep{opsd}.}
\end{table}